\documentclass[10pt,final,twocolumn]{IEEEtran}
\usepackage[utf8]{inputenc}
\usepackage{textcomp}
\DeclareSymbolFont{tildelow}{TS1}{cmr}{m}{n}
\DeclareMathSymbol{\tildelow}{0}{tildelow}{126}
\usepackage{bbm}
\usepackage{xcolor}
\usepackage{color}
\usepackage{graphicx}
\usepackage{epsfig}
\usepackage{subfigure}
\usepackage{amssymb}
\usepackage{amsbsy}
\usepackage{cite}
\usepackage{amsthm}
\usepackage{romannum}
\usepackage[noend]{algorithmic}
\usepackage{algorithm}

\newtheorem{theorem}{Theorem}
\newtheorem{lemma}{Lemma}
\newtheorem{remark}{Remark}
\newtheorem{definition}{Definition}

\newcommand{\beq}{\begin{equation}}
\newcommand{\eeq}{\end{equation}}
\newcommand{\bea}{\begin{array}}
\newcommand{\ena}{\end{array}}
\newcommand{\bds}{\begin {itemize}}
\newcommand{\eds}{\end {itemize}}
\newcommand{\bdf}{\begin{definition}}
\newcommand{\blm}{\begin{lemma}}
\newcommand{\edf}{\end{definition}}
\newcommand{\elm}{\end{lemma}}
\newcommand{\bthm}{\begin{theorem}}
\newcommand{\ethm}{\end{theorem}}
\newcommand{\bprp}{\begin{prop}}
\newcommand{\eprp}{\end{prop}}
\newcommand{\bcl}{\begin{claim}}
\newcommand{\ecl}{\end{claim}}
\newcommand{\bcr}{\begin{coro}}
\newcommand{\ecr}{\end{coro}}
\newcommand{\bquest}{\begin{question}}
\newcommand{\equest}{\end{question}}

\newcommand{\larrow}{{\larrow}}

\def\urltilda{\kern -.15em\lower .7ex\hbox{\~{}}\kern .04em}

\usepackage{float}
\usepackage[T1]{fontenc}
\usepackage{url}
\usepackage{ifthen}
\usepackage{cite}
\usepackage[cmex10]{amsmath} 

\renewcommand{\baselinestretch}{1.0}

\begin{document}
\title{Asymptotically Optimal Search for a Change Point Anomaly under a Composite Hypothesis Model}

\author{Liad Lea Didi, Tomer Gafni, Kobi Cohen (\emph{Senior Member, IEEE})
    \thanks{Liad Lea Didi, Tomer Gafni, and Kobi Cohen are with the School of Electrical and Computer Engineering, Ben-Gurion University of the Negev, Beer Sheva 8410501 Israel. Email: liadeli@post.bgu.ac.il, gafnito@post.bgu.ac.il, yakovsec@bgu.ac.il}
    \thanks{A short version of this paper that introduces the algorithm, and preliminary simulation results was presented at the annual Allerton Conference on Communication, Control, and Computing (Allerton) 2024 \cite{didi2024active}. In this journal version we include: (i) A detailed development and description of the algorithm; (ii) a deep theoretical analysis of the algorithm with detailed proofs; (iii) more extensive simulation results; and (iv) a detailed discussion of the results, and comprehensive discussion and comparison with the existing literature.}
    \thanks{This work was supported by the Israel Science Foundation under Grant 2640/20.}
	\vspace{-0.75cm}
}
\maketitle
\pagenumbering{arabic}

\begin{abstract}
We address the problem of searching for a change point in an anomalous process among a finite set of $M$ processes. Specifically, we address a composite hypothesis model in which each process generates measurements following a common distribution with an unknown parameter (vector). This parameter belongs to either a normal or abnormal space depending on the current state of the process. Before the change point, all processes, including the anomalous one, are in a normal state; after the change point, the anomalous process transitions to an abnormal state. Our goal is to design a sequential search strategy that minimizes the Bayes risk by balancing sample complexity and detection accuracy.

We propose a deterministic search algorithm with the following notable properties. First, we analytically demonstrate that when the distributions of both normal and abnormal processes are unknown, the algorithm is asymptotically optimal in minimizing the Bayes risk as the error probability approaches zero. In the second setting, where the parameter under the null hypothesis is known, the algorithm achieves asymptotic optimality with improved detection time based on the true normal state. Simulation results are presented to validate the theoretical findings.
\end{abstract}

\begin{IEEEkeywords}
Anomaly detection, change point detection, controlled sensing, active hypothesis testing. \vspace*{-0.2cm}
\end{IEEEkeywords}

\section{Introduction}
\label{sec:introduction}
We consider the problem of detecting a change point that occurs in an anomalous process among $M$ processes. At each time step, the decision maker observes a subset of \(K\) processes (\(1 \leq K \leq M\)). Using terminology from target search, we refer to these processes as cells, with the anomalous process after the change point being the target, located in one of the \(M\) cells. We adopt a composite hypothesis model, where the observations follow a distribution dependent on an unknown parameter (vector). When probing a specific cell, random continuous values are measured, drawn from a common distribution \(f\). This distribution \(f\) has an unknown parameter that belongs to either \(\Theta^{(0)}\) or \(\Theta^{(1)}\), depending on whether the target is absent or present, respectively. Before the change point, all cells are in a benign state, i.e., with parameters in \(\Theta^{(0)}\). After the change point, the distribution of the observations for one of the cells may change, with its parameter shifting to \(\Theta^{(1)}\). We assume the change point is an unknown deterministic parameter. The policy for selecting which subset of cells to probe at each time step must be designed jointly with the change detection algorithm. Therefore, the objective is to design a sequential search strategy that minimizes the Bayes risk, accounting for sample complexity and detection accuracy, by determining both the subset of cells to observe and when to terminate the search and declare the anomalous cell.

\subsection{Main Results}

The problem of searching for a change point anomaly, as considered in this paper, is closely related to both the sequential design of experiments and change point detection. Within the sequential design of experiments, search problems involving a finite set of processes require critical decisions about which processes to sample sequentially at each time step. This process continues until testing concludes, and the anomaly's location is reliably identified. This class of problems has been referred to in recent years as controlled sensing for hypothesis testing \cite{nitinawarat2013controlled} or active hypothesis testing \cite{naghshvar2013active}. These frameworks trace their origins to Chernoff’s seminal work on sequential design of experiments \cite{chernoff1959sequential}. Chernoff primarily addressed binary hypothesis testing, introducing a randomized test strategy, the Chernoff test, which is asymptotically optimal as the error probability approaches zero. Variations and extensions of this framework are discussed in Section \ref{ssec:related}. However, these approaches typically assume a fixed distribution setting, where the anomalous state is present from the outset.

In contrast, the problem addressed in this paper involves a target process that transitions from a benign to an anomalous state at an unknown change point. This aspect aligns the problem with change point detection, where the objective is to identify a change in the distribution as quickly as possible after it occurs, while maintaining a low false alarm rate to avoid premature declarations. The CUSUM algorithm, introduced in \cite{11e3eed7-cdc6-3579-8127-4296922d24e9}, serves as a foundational method in change point detection for single processes, achieving optimality under certain performance criteria. Variants of this problem are also discussed in Section \ref{ssec:related}. Despite these connections, the specific problem of searching for a change point anomaly across multiple processes under a composite hypothesis model considered here has not been analyzed before. Below, we summarize the main contributions of this work. 

First, in terms of algorithm development, We propose a novel algorithm, Searching for Change Point Anomaly (SCPA), to address the problem at hand. The algorithm is deterministic with low-complexity implementations and operates through three distinct phases: Exploration, exploitation, and sequential testing. In the exploration phase, all cells are probed in a round-robin fashion to estimate the unknown parameter. During the exploitation phase, probing focuses on the cell where the change point is most likely to occur. Finally, in the sequential testing phase, this cell is sampled until sufficient confidence is reached to declare it as anomalous. A notable challenge arises from early observations in the pre-change regime, which could delay detection. This is mitigated by strategically resetting the statistic measure. Detailed descriptions of the algorithm are provided in Section \ref{sec: The SCPA Algorithm}.

In terms of theoretical performance analysis, we establish the asymptotic optimality of SCPA as the error probability approaches zero. Traditional change point analysis often restricts the change point $\tau_c$ by assuming prior distributions on $\tau_c$ or adopting weaker performance metrics, such as fixed, non-vanishing error rates. Here, we introduce a novel approach to tackle the problem by setting $\tau_c$ in a manner that ensures robust guarantees for vanishing error probabilities. Specifically, we impose an asymptotic constraint on $\tau_c$, enabling it to be arbitrarily large while ensuring it occurs within a timeframe that does not excessively delay detection. This approach enables strong guarantees on the vanishing of error probabilities, aligning with the classical Chernoff's sequential experimental design \cite{chernoff1959sequential}, and addresses gaps in existing approaches where globally optimal solutions are unavailable under general settings \cite{veeravalli2014quickest, tartakovsky2005asymptotic}. Current modifications often use techniques like increasing curved boundaries \cite{borovkov1999asymptotically} or constraining local false alarm probabilities \cite{tartakovsky2005asymptotic}. In contrast, our analysis provides explicit conditions under which SCPA achieves asymptotic optimality. This setting is particularly applicable to real-world scenarios, such as cyber anomaly detection, where algorithms operate during intervals where anomalies are likely to manifest (e.g., before corrective actions like patch installations are implemented). More generally, it applies to systems operating within time intervals where the focus is on ensuring vanishing error probabilities within periods after the change point. This robust framework contrasts with approaches accepting fixed error rates across the entire time horizon.

We demonstrate the asymptotic optimality of SCPA in two scenarios: (i) When the distributions of normal and anomalous processes are unknown, and (ii) when the parameter value under the null hypothesis is known and identical across all non-anomalous processes. In both cases, our algorithm matches the performance of the active hypothesis testing setting in \cite{hemo2020searching}, which is a special case of our problem where the anomaly is present from the start without temporal change. Finally, we present simulation results that validate our theoretical findings and demonstrate the practical effectiveness of the SCPA algorithm.

\subsection{Related Work}
\label{ssec:related}

As highlighted earlier, our problem is closely related to both sequential experimental design and change point detection. Sequential experimental design, originally introduced by Chernoff \cite{chernoff1959sequential}, has been extensively studied and expanded upon in works such as \cite{naghshvar2013active, nitinawarat2013controlled, cohen2015asymptotically, kaspi2017searching, song2017asymptotically, vaidhiyan2017learning, huang2018active, zhong2019deep, gurevich2019sequential, hemo2020searching, wang2020information, gafni2021searching, tajer2021active, lambez2021anomaly, tsopelakos2022sequential, tsopelakos2023asymptotically, gafni2023anomaly, szostak2024deep}. The specific challenge of anomaly detection across multiple processes has been addressed in \cite{cohen2015active, lambez2021anomaly, huang2018active, gurevich2019sequential, tsopelakos2022sequential, tsopelakos2023asymptotically, gafni2023anomaly, cohen2015asymptotically, song2017asymptotically, hemo2020searching, vaidhiyan2017learning, zhong2019deep}. The detection of abnormal processes over densities with an unknown parameter was explored in \cite{cohen2015asymptotically}, focusing on independent process states across cells. In \cite{cohen2015active}, anomaly detection among $M$ processes was studied, where $K$ processes could be observed at each time step. Observations followed one of two known distributions, depending on whether the process was normal or abnormal. A deterministic search policy was proven asymptotically optimal in minimizing detection time under an error probability constraint. An extension of \cite{cohen2015active} was considered in \cite{hemo2020searching}, where observations followed a common distribution with an unknown parameter, varying based on the process state. A sequential search strategy achieving asymptotic optimality was developed. For Poisson point processes with unknown rates, \cite{vaidhiyan2017learning} established an asymptotically optimal policy for single-location probing, requiring a randomized selection rule and linear exploration time. Non-parametric detection over finite observation spaces was studied in \cite{nitinawarat2015universal}, showing asymptotic optimality with logarithmic exploration time when the null distribution is known. Extensions to M-ary and composite hypothesis testing for single processes were investigated in \cite{schwarz1962asymptotic,lai1988nearly}. In contrast to these anomaly detection frameworks with fixed distribution settings, our setting addresses a dynamic scenario involving a change point where the anomalous process transitions between states. Additional works on sequential detection across independent processes include \cite{5961845, 6482259,6736855, cohen2015asymptotically,7028432,7518609}.

This work also relates to the field of change point detection. Page \cite{page1954continuous} introduced CUSUM charts, which leverage cumulative data for enhanced detection efficiency. Shiryaev \cite{shiryaev1963optimum} developed a Bayesian framework minimizing average detection delay with known geometric priors. Lorden \cite{lorden1971procedures} proposed a minimax framework optimizing worst-case detection delay, with CUSUM achieving asymptotic optimality. Subsequent improvements to Lorden's criterion were presented in \cite{moustakides1986optimal, ritov1990decision}. Pollak \cite{pollak1985optimal} introduced a single-maximization criterion, showing the asymptotic optimality of Shiryaev-Roberts and CUSUM algorithms \cite{lai1988nearly}. Non-Bayesian quickest detection was explored in \cite{bayraktar2015byzantine}, demonstrating CUSUM's optimality. In \cite{10423411}, the problem of detecting distribution changes in random vector sequences was considered, addressing unknown post-change parameters influenced by control actions, requiring both a stopping rule and a sequential control policy. \cite{10129981} addressed streaming data changing from a known to an unknown parametric distribution, combining CUSUM with post-change parameter estimation. Shiryaev \cite{doi:10.1137/1108002} examined change detection in a single process under a Bayesian loss framework, while Lorden \cite{77dcb352-3962-3a93-92e9-f809dbd6d500} and Pollak \cite{5278acc0-9417-3730-988a-4774fe31a3f7} proposed minimax criteria for i.i.d. observations with known pre- and post-change distributions but unknown deterministic change points. Asymptotic performance of Shiryaev's procedures was studied in \cite{doi:10.1137/S0040585X97981202}. Extensions and variations appear in \cite{veeravalli2014quickest, xie2021sequential, 8410438, puzanov2018deep}. Multi-stream settings were addressed in \cite{9410585}, where streams share a known distribution pre-change, and one stream diverges to a different known distribution post-change. \cite{doi:10.1080/07474946.2023.2187417} generalized to exponential distributions with a common pre-change parameter, where one stream shifts to an unknown parameter. However, these studies focused on declaring a change without pinpointing the affected stream.

\section{System Model and Problem Formulation}
\label{sec:problem}

We begin this section by presenting the system model, followed by formulating the objective.

    \subsection{System Model}
    \label{sec: System Model}

We consider the problem of detecting a change point in an anomalous process within a finite set of \( M \) processes, where $M$ is large yet finite. Initially, before the change point, all processes—including the anomalous one—are in a normal state. After the change point, the anomalous process transitions to an abnormal state. Specifically, the system consists of $M$ discrete time stochastic processes $\{X_n^m\}_{n \geq 1}, 1 \leq m \leq M$, with $X_n^m$ taking real values at times $n=1, 2, ...$. The distribution of the anomalous process changes at an unknown deterministic change point \( \tau_c \). If process \( m \) is the anomalous one, we denote this as hypothesis \( H_m \) being true. Drawing an analogy from target search, we often refer to the anomalous process as the target, which can be located in any of the \( M \) cells. 

We focus on a composite hypothesis setting, where the observation distribution depends on an unknown parameter (vector). Let \( \theta_m \) denote the unknown parameter specifying the observation distribution for cell \( m \). When cell \( m \) is observed at time \( n \), an observation \( y_m(n) \) is drawn independently from a common density \( f(y | \theta_m) \), where \( \theta_m \in \Theta \) and \( \Theta \subseteq \mathbb{R} \) represents the parameter space for all cells. We define two disjoint open subsets of \( \Theta \): \( \Theta^{(0)} \) and \( \Theta^{(1)} \), representing the parameter spaces for the non-anomalous distribution and the anomalous distribution, respectively. Specifically, under the non-anomalous distribution, \( \theta_m \in \Theta^{(0)} \), while under the anomalous distribution, \( \theta_m \in \Theta^{(1)} \). Under hypothesis \( H_m \), the observations \( \{y_j(n)\}_{n \geq 1}, \forall j \neq m \), are drawn independently from \( f(y|\theta_j) \) with \( \theta_j \in \Theta^{(0)} \). For cell \( m \), the observations \( y_m(1), y_m(2), \ldots, y_m(\tau_c-1) \) are drawn independently from \( f(y|\theta_m) \) with \( \theta_m \in \Theta^{(0)} \), and the observations \( y_m(\tau_c), y_m(\tau_c+1), \ldots \) are drawn independently from \( f(y|\theta_m) \) with \( \theta_m \in \Theta^{(1)} \). The prior probability that \( H_m \) is true is denoted by \( \pi_m \), where \( \sum_{m=1}^{M} \pi_m = 1 \). To avoid trivial solutions, we assume \( 0 < \pi_m < 1 \) for all \( m \). At each time step, only \( K \) cells (\( 1 \leq K \leq M \)) can be observed. Let \( \mathbf{P}_m \) denote the probability measure under hypothesis \( H_m \), and let \( \mathbf{E}_m \) represent the expectation operator with respect to the measure \( \mathbf{P}_m \).

The stopping time \( \tau \) is defined as the time when the decision-maker concludes the search by declaring the target's location. Let \( \delta \in \{1, 2, \ldots, M\} \) be a decision rule, where \( \delta = m \) indicates that the decision-maker declares \( H_m \) to be true. Define \( \phi(n) \in \{1, 2, \ldots, M\}^K \) as the selection rule specifying the \( K \) cells chosen for observation at time \( n \). The time-series vector of selection rules is denoted by \( \pmb{\phi} = (\phi(n), n = 1, 2, \ldots) \). Let \( \textbf{y}_{\phi(n)}(n) \) represent the vector of observations obtained from the selected cells \( \phi(n) \) at time \( n \), and let \( \textbf{y}(n) = \{\phi(t), \textbf{y}_{\phi(t)}(t)\}_{t=1}^{n} \) represent the collection of all cell selections and corresponding observations up to time \( n \). A deterministic selection rule \( \phi(n) \) at time \( n \) is a mapping from \( \textbf{y}(n-1) \) to \( \{1, 2, \ldots, M\}^K \). Alternatively, a randomized selection rule \( \phi(n) \) maps \( \textbf{y}(n-1) \) to probability mass functions over \( \{1, 2, \ldots, M\}^K \). 

An admissible strategy \( \Gamma \) for the sequential change point anomaly detection problem is characterized by the tuple \( \Gamma = (\tau, \delta, \pmb{\phi}) \). Note that we seek a search strategy that does not rely on knowledge of \( \tau_c \) and assumes no prior knowledge about it.

\subsection{Objective}
\label{sec:Objective}

We begin by defining the error probabilities associated with a sequential decision strategy \( \Gamma \). In classical change point detection involving a single process, the decision rule declares a change has occurred once sufficient evidence is gathered. In this setting, only false alarm events can occur, which correspond to declaring a change before it actually happens. After the change point, all subsequent decisions are correct. The objective in this case is to strike a balance between minimizing the detection delay, i.e., the time elapsed between the actual change point and the declaration time, and minimizing the false alarm probability.

In contrast, for anomaly change point detection involving multiple processes, as considered in this paper, the objective is not only to detect whether a change has occurred but also to identify the anomalous process in which the change occurred. As a result, errors arise from two types of events, depending on whether the change point has occurred at or before the stopping time \( \tau \): false alarms and missed detections. The false alarm probability is the probability of declaring an anomaly in one of the $M$ cells before the change point occurs (when all processes are still in their normal states). Conversely, the missed detection probability is the probability of incorrectly identifying a non-anomalous cell as the target after the change point has occurred (when one process is in an anomalous state, but the decision-maker fails to detect it and instead declares another process as anomalous).

Formally, under strategy $\Gamma$ and hypothesis $H_{m}$, the false alarm probability is given by: 
    \begin{equation}
    \label{eq : FA error probability definition}
        \textbf{P}_{m}(\tau < \tau_c \, | \, \Gamma \, ; \,\tau_c) \, ,
    \end{equation}
and the miss detection probability is given by: 
    \begin{equation}
    \label{eq : MD error probability definition}
        \textbf{P}_{m}(\delta \neq m , \, \tau \geq \tau_c \, | \, \Gamma \, ; \,\tau_c) \, .
    \end{equation}
Hence, the error probability under hypothesis $H_{m}$ is given by: 
    \begin{equation}
    \label{eq : MD error probability definition}
        \alpha_{m} (\Gamma \, ; \,\tau_c) = \textbf{P}_{m}(\delta \neq m , \, \tau \geq \tau_c \, | \, \Gamma \, ; \,\tau_c) + \textbf{P}_{m}(\tau < \tau_c \, | \, \Gamma \, ; \,\tau_c) .
    \end{equation}
Finally, the overall error probability under strategy $\Gamma$ is given by:
    \begin{equation}
    \label{eq : error probability definition}
        P_{e}(\Gamma \, ; \,\tau_c) = \sum _{m=1}^{M} \pi_{m}\alpha_{m} (\Gamma \, ; \,\tau_c) \, .
    \end{equation}
Note that in the special case of $M=1$, i.e., a single anomalous process, missed detection events do not occur. Consequently, the error probability reduces to the false alarm probability (\ref{eq : FA error probability definition}).

Let $\textbf{E}(\left(\tau -\tau_c\right)^{+}|  \Gamma  ;\tau_c ) = \sum _{m=1} ^{M} \pi _{m} \textbf{E}_{m}( \left(\tau - \tau_c\right)^{+} | \Gamma ;\tau_c )$ be the average detection delay under $\Gamma$, where $x^{+} = \max\{x,0\}$. We adopt the Bayes risk framework, as utilized in both active hypothesis testing and change point detection formulations \cite{chernoff1959sequential,cohen2015active, lai1988nearly,wald2004sequential,hemo2020searching, 5961845}, by assigning a cost of $c$ for each observation made after the change point and a loss of $1$ for a wrong declaration. The risk under strategy $\Gamma$ when hypothesis $H_{m}$ is true and change point $\tau_c$, is thus given by:
    \begin{equation}
    \label{eq: Bayes risk}
        R_{m}(\Gamma \, ; \,\tau_c) \triangleq \alpha_{m}(\Gamma \, ; \,\tau_c)+c\: \textbf{E}_{m}\left( \left(\tau - \tau_c\right)^{+} | \Gamma ;\tau_c \right).
    \end{equation}
The average risk is given by:
\begin{equation}
\label{eq: average Bayes risk}
\begin{array}{l}
\displaystyle
R(\Gamma \, ; \,\tau_c ) = \sum_{m=1}^{M} \pi_{m}R_{m}(\Gamma \, ; \,\tau_c)
\vspace{0.1cm}\\ \displaystyle\hspace{2cm}
=P_{e}(\Gamma \, ; \,\tau_c)+c\: \textbf{E}\left( \left(\tau - \tau_c\right)^{+} \, | \, \Gamma \, ;\tau_c \right).
\end{array}    
\end{equation}

The objective is to find a strategy $\Gamma$ that minimizes the Bayes risk $R(\Gamma)$ over all possible realizations of $\tau_c$ without knowledge of its value:
\begin{equation}
    \label{eq:inf R}
        \underset{\Gamma}{\inf} \; R(\Gamma \, ; \,\tau_c).
    \end{equation}

\section{The Search Algorithm for Change Point Anomaly (SCPA)}
\label{sec: The SCPA Algorithm}

In this section we present the Search algorithm for Change Point Anomaly (SCPA)
for solving (\ref{eq:inf R}). To streamline the presentation, we describe the algorithm for \( K = 1 \) (i.e., sampling one cell at each time slot), which is common in dynamic search problems (e.g., \cite{6844162, nitinawarat2015universal,hemo2020searching} and subsequent studies). Later, in Section \ref{ssec:implementation}, we explain its extension to \( K > 1 \). We also assume that the parameter space $\Theta$ is finite, consistent with the assumption made in \cite{chernoff1959sequential} and subsequent studies. Later, in Section \ref{sec:performance}, we will provide asymptotic analysis of the algorithm performance. 

Let us provide below notations used in the algorithm description and throughout the paper. For every $\theta, \varphi \in \Theta$ we define
        \begin{equation}
             \ell_{m}^{(\theta,\varphi)}(t) \triangleq \log \frac{f(y_{m}(t)|\theta)}{f(y_{m}(t)|\varphi)},
        \end{equation}
as the log-likelihood ratio (LLR), testing $\theta$ against $\varphi$ at time $n$. The term 
\begin{equation}
D(\theta||\varphi) \triangleq \textbf{E}_{f(y|\theta)}\left[ \log \frac{f(y|\theta)}{f(y|\varphi)}\right]   
\end{equation}
denotes the Kullback-Leibler (KL) divergence between the two distributions, $f(y|\theta)$ and $f(y|\varphi)$.

\subsection{Description of the Algorithm}

The SCPA algorithm is structured into exploration and exploitation epochs, as described below.\vspace{0.2cm}

    \begin{enumerate}
        \item \textit{Phase 1: Sequential Round-Robin Exploration:} Observe the cells sequentially in a round-robin manner, collecting $y_{m}(n)$ (the observation from cell $m$ at time $n$), and estimate the cell parameter using the last $N$ observations, $\Tilde{\textbf{y}}_{m}\triangleq \{y_{m}(r_1),...,y_{m}(r_N)\}$, through an unconstrained maximum likelihood estimator (MLE)
        \begin{equation}
        \label{eq: MLE estimation in exploration}
            \hat{\theta}_{m}(n) \triangleq \arg \max _{\theta \in \Theta} f(\Tilde{\textbf{y}}_{m}|\theta).
        \end{equation}
        Here, $N$ is a predefined parameter, and its selection will be discussed later in the algorithm description.\\
        If $|\{m : \hat{\theta}_{m}(n) \not\in \Theta^{(0)}\}|=1$ (i.e., there is exactly one cell suspected of having experienced a change point), proceed to Phase 2. Otherwise, return to Phase 1.\vspace{0.2cm}
        \item \textit{Phase 2: Exploitation with Targeted Parameter Estimation:} Let $T$ denote the most recent time at which Phase 1 was exited, and let $\hat{m}(n)=\{m : \hat{\theta}_{m}(n) \not\in \Theta ^{(0)}\}$ represent the index of the single cell whose MLE falls outside $\Theta^{(0)}$. Observe cell $\hat{m}(n)$. Based on the observations $\overline{\textbf{y}}_{\hat{m}(n)}(n) \triangleq \{y_{\hat{m}(n)}(T+1),...,y_{\hat{m}(n)}(n)\}$ estimate the cell parameter using the unconstrained MLE:
        \begin{equation}
        \label{eq: MLE estimation in exploitation}
            \hat{\theta}_{\hat{m}(n)}(n) \triangleq \arg \max _{\theta \in \Theta} f(\overline{\textbf{y}}_{\hat{m}(n)}(n)|\theta).
        \end{equation}
        If $\hat{\theta}_{\hat{m}(n)}(n) \in \Theta^{(0)}$, return to Phase 1. Otherwise, proceed to Phase 3.\vspace{0.2cm}
        \item \textit{Phase 3: Sequential Testing:} Let 
        \begin{equation}
        \label{eq: MLE estimation in sequential testing}
            \hat{\theta}_{m}^{(0)}(n) \triangleq \arg \max _{\theta \in \Theta^{(0)}} f(\overline{\textbf{y}}_{m}(n)|\theta)
        \end{equation}
        be the constrained MLE for cell $m$ to be in a normal state, and
\begin{equation}
\begin{array}{l}
\label{eq: ALLR at sequential testing}
\displaystyle
S_{m}(n) \triangleq  \sum_{t=T+2}^{n} \ell_{m}^{(\hat{\theta}_{m}(t-1),\hat{\theta}_{m}^{(0)}(n))}(t) \vspace{0.2cm}\\ \displaystyle
\hspace{1cm}
=\sum_{t=T+2}^{n} \log \frac{f(y_{m}(t)| (\hat{\theta}_{m}(t-1))}{f(y_{m}(t)| \hat{\theta}_{m}^{(0)}(n))}
\end{array}
\end{equation}
be the sum of the adaptive LLR (SALLR) of cell $m$ at time n used to confirm hypothesis $m$. 

If $ S_{\hat{m}(n)}(n) \geq - \log(c)$ stop the test and declare $\delta=\hat{m}(n)$ as the anomalous cell, otherwise go to Phase 2.
    \end{enumerate}

\subsection{Insights and Discussion on the SCPA Algorithm}
\label{ssec:implementation}    

\subsubsection{Considerations and Implementation Details} Note that the SCPA algorithm is deterministic and transitions between phases based on the current state as determined by the MLE. In the exploration phase, we estimate the cell parameter using only the most recent $N$ observations. The value of $N$ is a predefined parameter, which can be chosen freely. Since we only use a finite number of observations, after the change point, we exclude benign observations from the abnormal cell that occurred before the change point. This approach is particularly beneficial when the change point happens after a long time, as it reduces the need to accumulate too many abnormal observations to counteract the bias from earlier normal ones, thus ensuring that our performance remains unaffected by the timing of the change point. Selecting $N$ too small may result in instability, causing the algorithm to repeatedly enter and exit Phase 1, while a larger $N$ may lead to prolonged stays in Phase 1 after the change point. For simplicity, we chose $N = 1$ in our theoretical and numerical evaluations, though any fixed natural constant can be used theoretically.

In Phases 2 and 3, we avoid using all the normal observations of the anomalous cell to prevent performance degradation, while ensuring enough observations from the target cell are accumulated to confirm the correct hypothesis. Note that in Phase 3, the summation begins at $T+2$ rather than $T+1$, ensuring it relies on observations from $T+1$ to $n$, excluding $T$. This adjustment accounts for the parameter estimation in the numerator and arises from technical considerations necessary to establish the theoretical asymptotic performance presented in the appendix.\vspace{0.1cm}

\subsubsection{Tackling Multi-Process Probing and Detection of Multiple Anomalies} Next, we discuss the generalization of the SCPA algorithm to accommodate the case of $K \geq 1$. The extension works as follows: In Phase 1, at each time step, $K$ cells are sequentially probed in a round-robin fashion in groups of $K$. The transition to Phase 2 occurs under the same condition as before, when exactly one cell is suspected of having an anomaly. In Phase 2, the remaining cells are ranked by their SALLR values, selecting the next $K-1$ cells with the highest likelihood of being anomalous. The suspected cell and these $K-1$ cells are then probed. If the suspected cell’s parameter estimate falls within the normal range, or if any of the $K-1$ cells exhibit parameter estimates outside the normal range (based on observations collected since entering this phase), we return to Phase 1. Otherwise, we proceed to Phase 3. In Phase 3, the algorithm compares the SALLR of the suspected cell with the second-highest SALLR. If the difference is greater than or equal to $-\log c$, the suspected cell is declared abnormal, and the test stops. Otherwise, the process returns to Phase 2.

Next, we address the extension of SCPA to handle multiple (say $L$) anomalies. The implementation of Phase 1 remains unchanged, but the algorithm proceeds to Phase 2 when exactly $L$ cells are suspected of being anomalous. In Phase 2, these $L$ cells are sequentially observed. If any of the $L$ cells no longer appear anomalous based on their parameter estimates, the algorithm returns to Phase 1. Otherwise, after each observation in Phase 2, the SALLR values of the cells are evaluated. When the maximum SALLR exceeds the threshold, the corresponding cell is declared anomalous. The algorithm then continues observing the remaining $L-1$ suspected cells in Phase 2 until all $L$ anomalies are detected.

\subsubsection{Using Generalized LLR} In Phase 3, the sum of the adaptive LLR from equation (\ref{eq: ALLR at sequential testing}) is used to evaluate the suspected cell. Alternatively, the sum of the Generalized LLR (GLLR) can be applied in a similar manner:
\begin{equation}
\begin{array}{l}
\label{eq: GLLR at sequential testing}
\displaystyle
S_{m}(n) \triangleq  \sum_{t=T+2}^{n} \ell_{m}^{(\hat{\theta}_{m}(n),\hat{\theta}_{m}^{(0)}(n))}(t) \vspace{0.2cm}\\ \displaystyle
\hspace{1cm}
=\sum_{t=T+2}^{n} \log \frac{f(y_{m}(t)| (\hat{\theta}_{m}(n))}{f(y_{m}(t)| \hat{\theta}_{m}^{(0)}(n))}.
\end{array}
\end{equation}
The ALLR is advantageous for bounding the error probability effectively using martingale techniques. However, it depends on parameter estimates derived from a limited number of initial observations, which remain fixed and cannot be updated even as additional data is collected. On the other hand, the GLLR \eqref{eq: GLLR at sequential testing} is known to achieve better empirical performance, as it updates parameter estimates with the latest observations. Despite this practical advantage, the GLLR lacks the theoretical performance guarantees provided by the ALLR.

\section{Performance Analysis}
\label{sec:performance}

In this section, we provide an asymptotic analysis of the algorithm's performance. The analysis assumes the setting of $K=1$ and a single anomaly, which are common assumptions in dynamic search problem analysis (e.g., \cite{6844162, nitinawarat2015universal,hemo2020searching} and subsequent studies). We begin in Section \ref{ssec:Anomaly Detection Without Side Information} by examining the case where no additional information about the process states is available—that is, both the pre-change and post-change distribution parameters are unknown. In Section \ref{ssec:Anomaly Detection Under a Known Model of Normality}, we consider the case where the parameter under the null hypothesis is known and identical across all non-anomalous processes.

In traditional change point analysis, deriving optimality properties often involves restricting the occurrence of the change point, $\tau_c$. This is typically done by assuming a prior distribution on $\tau_c$ or adopting weaker performance measures, such as allowing a fixed error rate that does not vanish asymptotically. In contrast, this work introduces an assumption for $\tau_c$ that enables strong guarantees on the vanishing of error probabilities, aligning with the classical frameworks of sequential experimental design by Chernoff \cite{chernoff1959sequential}. Specifically, we impose an asymptotic constraint on $\tau_c$, allowing it to be arbitrarily large while ensuring its occurrence is not excessively delayed relative to the detection time.

This assumption is particularly relevant in practical anomaly detection systems. For example, in cyber anomaly detection, algorithms often operate during suspicious intervals where anomalies, such as malicious or undesired behavior, are expected to manifest before corrective actions (e.g., patch installations) take effect. More generally, this assumption applies to systems operating within time intervals, where the focus is on ensuring vanishing error probabilities within intervals where the change has occurred (i.e., when the assumption on the asymptotic bound on $\tau_c$ holds), rather than accepting fixed error rates over the entire time horizon. This approach provides a robust framework for maintaining strong performance guarantees under realistic operational conditions. The asymptotic constraint on $\tau_c$ assumed throughout the analysis in this section is formally defined as follows:

\noindent
\textbf{Assumption 1:} Let $c$ be the cost per observation defined in \eqref{eq: Bayes risk}. Then, the change point $\tau_c$ is of order $O\left((-\log c)^{1-\delta}\right)$ for some $0<\delta<1$.

This assumption implies that $\tau_c$ grows slower than $-\log c$, specifically, $o\left(-\log c\right)$. Later, we demonstrate that the detection time is of order $O\left(-\log c)\right)$, ensuring that $\tau_c$ does not occur too late relative to the detection time as $c$ approaches zero.

    \subsection{change point Anomaly Search without Side Information}
    \label{ssec:Anomaly Detection Without Side Information}
    We start by analyzing the SCPA algorithm in the setting where the parameters of both the anomalous and non-anomalous distributions are unknown. The following theorem establishes its asymptotic optimality:
    \begin{theorem}
    \label{thorem 1}
        Let $R^{*}$ and $R(\Gamma \, ; \,\tau_c)$ be the Bayes risks under the SCPA algorithm and any other policy $\Gamma$, respectively. Then, the Bayes risk satisfies:\footnote{The notation $f \sim g$ as $c \rightarrow 0$ refers to $\lim_{c \rightarrow 0} f/g=1$.}
            \begin{eqnarray*} \begin{array}{l}\displaystyle R^* \;\sim \; \frac{-c\log c}{D(\theta ^{(1)})}\;\sim \;\inf _{\Gamma }\;{R(\Gamma \, ; \,\tau_c)} \;\;\; \text{as} \;\;\; c\rightarrow 0, \end{array} \end{eqnarray*}
            where $\displaystyle D(\theta ^{(1)}) \triangleq \min\nolimits_{\varphi \in \Theta ^{(0)}} D(\theta ^{(1)}||\varphi)$.
    \begin{proof}

The proof can be found in Appendix \ref{subsection proof of theorem 1}. The asymptotic optimality of the SCPA algorithm relies on several key results demonstrated in the Appendix. First, we establish that the algorithm achieves an error probability of order \( O\left(c \cdot (-\log c)^{1-\delta}\right) \). Second, we show that the post-change exploration time is of order \( O(1) \). The bounded post-change exploration time of the SCPA algorithm is particularly noteworthy, as it contrasts sharply with the logarithmic exploration time commonly observed in many active search strategies (see, for example, \cite{nitinawarat2015universal, cohen2015asymptotically}). Additionally, recall that the asymptotic constraint on \( \tau_c \) results in a pre-change time of order \( o\left(-\log c\right) \). Finally, we demonstrate that the post-change detection time asymptotically satisfies $\sim \frac{-c\log c}{D(\theta^{(1)})}, \text{as } c \to 0.$

Taken together, these results yield the asymptotic Bayes risk stated earlier. Since \( \frac{-c\log c}{D(\theta^{(1)})} \) serves as a lower bound on the Bayes risk of any algorithm, these findings collectively establish the asymptotic optimality of SCPA. For a detailed proof, see Appendix \ref{subsection proof of theorem 1}.
    \end{proof}
    \end{theorem}

    \subsection{change point Anomaly Search Under a Known Model of Normality}
    \label{ssec:Anomaly Detection Under a Known Model of Normality}
Here, we assume that the parameter under the null hypothesis, $\theta = \theta^{(0)} \in \Theta^{(0)}$, is known and equal for the non-anomalous distribution across all cells. This setup models many anomaly detection scenarios, where the decision maker can infer the pre-change distribution parameters, which represent the typical distribution before the change, while there is uncertainty regarding the distribution after the change point (i.e., under abnormal conditions). To utilize this information, we substitute the estimation of the parameter $\theta^{(0)}$ in the SALLR statistics (\ref{eq: ALLR at sequential testing}) with its known true value: 
    \begin{equation}
    \begin{array}{l}
        \label{eq: ALLR with side info}
        \displaystyle S_{m}(n) \triangleq \sum_{t=T+2}^{n} \log{\frac{f(y_{m}(t)|\hat{\theta}_{m}(t-1))}{f(y_{m}(t)|\theta^{(0)})}}.
    \end{array}
    \end{equation}
The following theorem establishes the asymptotic optimality of SCPA algorithm in this setting.
    \begin{theorem}
    \label{thorem 2}
         Let $R^{*}$ and $R(\Gamma \, ; \,\tau_c)$ be the Bayes risks under the SCPA algorithm and any other policy $\Gamma$, respectively. Then, the Bayes risk satisfies:
\begin{equation} 
\begin{array}{l}\displaystyle R^* \;\sim \; \frac{-c\log c}{D(\theta ^{(1)}||\theta^{(0)})}\;\sim \;\inf _{\Gamma }\;{R(\Gamma \, ; \,\tau_c)} \;\;\; \text{as} \;\;\; c\rightarrow 0.
\end{array} 
\end{equation}
    \end{theorem}
    \begin{proof}
        The proof is given in Appendix \ref{subsection proof of theorem 2}.
    \end{proof}

We note that having knowledge of the parameter for the non-anomalous distribution enhances the algorithm's performance. This improvement is evident from the fact that $D(\theta ^{(1)}||\theta^{(0)}) > D(\theta^{(1)})$, which implies that the Bayes risk stated in Theorem \ref{thorem 2} is smaller than the risk presented in Theorem \ref{thorem 1}.  

\section{Numerical Examples}

We validate the theoretical findings of SCPA through numerical examples. In our simulations, we modeled a single anomaly occurring in one of the \( M \) cells after the change point, with the following setup: The a priori probability of the target being present in cell \( m \) was set to \( \pi_m = 1/M \) for all \( 1 \leq m \leq M \). When cell \( m \) was observed at time \( n \), an observation \( y_m(n) \) was independently drawn from either the distribution \( \exp(\theta^{(0)}) \) or \( \exp(\theta^{(1)}) \), depending on whether the target was absent or present, respectively.

We examine the cases for both \(\tau_c = 0\) and \(\tau_c = 70\), with the results shown in Fig.~\ref{fig:avg_delay_vs_minus_logc_tau_c_0} and Fig.~\ref{fig:avg_delay_vs_minus_logc_tau_c_70}, respectively, for the average detection delay, and in Fig.~\ref{fig:avg_delay_vs_p_error_tau_c_0} and Fig.~\ref{fig:avg_delay_vs_p_error_tau_c_70}, respectively, for the log error probability. The results indicate similar outcomes for both realizations of the change point, as it does not influence the detection time. This observation is consistent with our theoretical findings. Additionally, we observe a linear relationship between the detection time and \(-\log c\), which further aligns with the theoretical findings.

\vspace{-0.5cm}
\begin{figure}[h]
    \centering
    \includegraphics[scale=0.4]{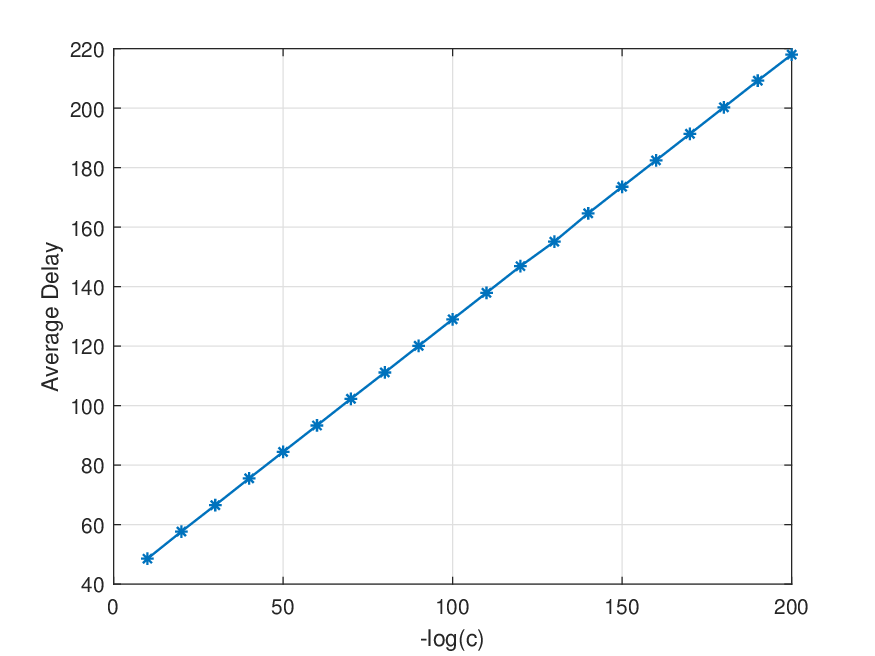}
    \caption{The average detection delay as a function of $-\log c$, with $M=5$, $\Theta^{(0)}= \{0.1N , 1\leq N \leq 10  \}$ , $\Theta^{(1)}= \{N , 2\leq N \leq 10  \}$ and $\tau_c=0$.} \label{fig:avg_delay_vs_minus_logc_tau_c_0}
\end{figure}

\begin{figure}[h]
    \centering
    \includegraphics[scale=0.4]{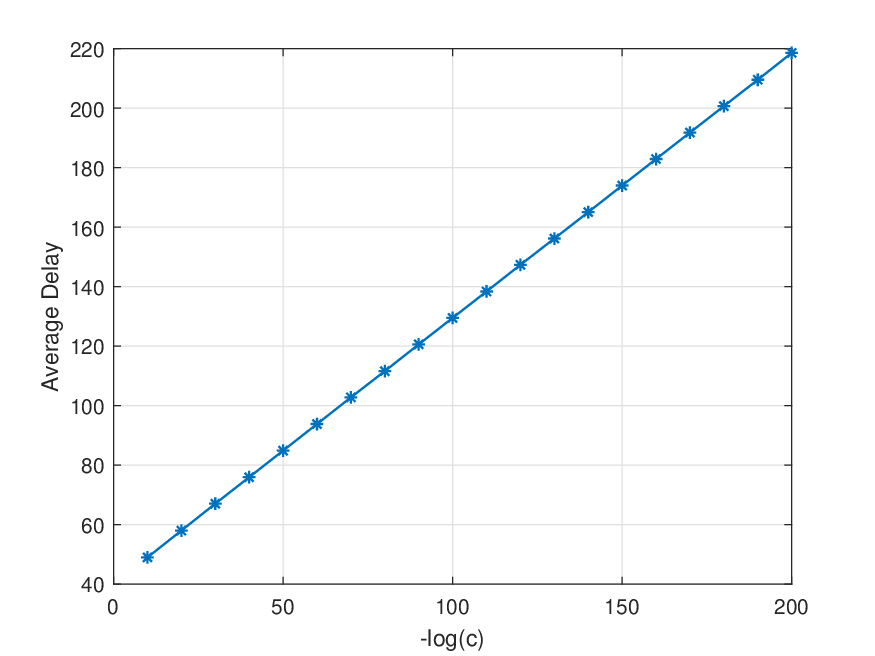}
    \caption{The average detection delay as a function of $-\log c$, with $M=5$, $\Theta^{(0)}= \{0.1N , 1\leq N \leq 10  \}$ , $\Theta^{(1)}= \{N , 2\leq N \leq 10  \}$ and $\tau_c=70$.} \label{fig:avg_delay_vs_minus_logc_tau_c_70}
\end{figure}

\begin{figure}[h]
    \centering
    \includegraphics[scale=0.4]{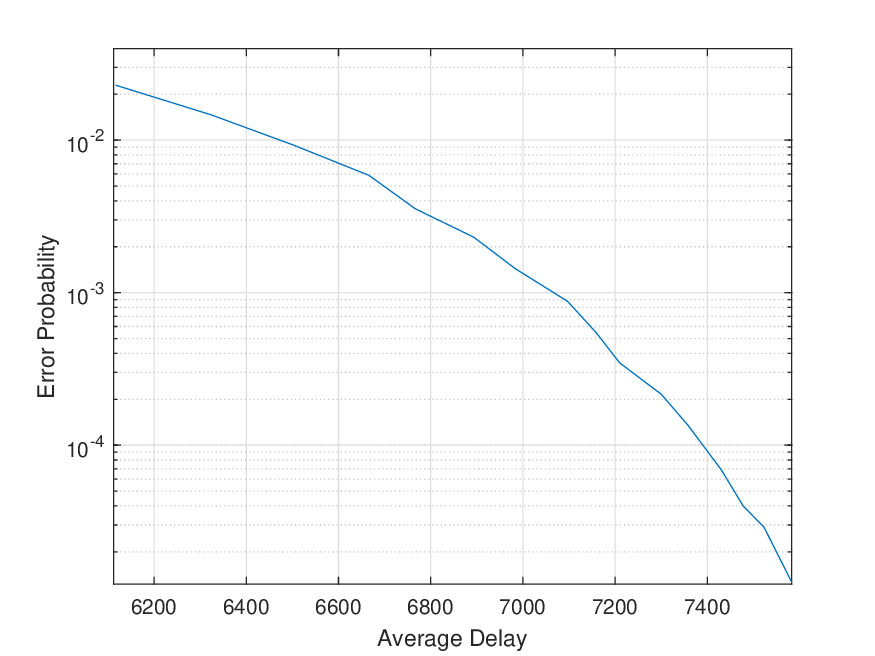}
    \caption{{The error probability as a function of the average detection delay, with $M=5$, $\Theta^{(0)}= \{1+0.1N , 0\leq N \leq 10  \}$ , $\Theta^{(1)}= \{0.5+0.1N , 0\leq N \leq 4  \} \cup \{2.1+0.1N , 0\leq N \leq 4  \}$ and $\tau_c=0$.}} \label{fig:avg_delay_vs_p_error_tau_c_0}
\end{figure}

\begin{figure}[h]
    \centering
    \includegraphics[scale=0.4]{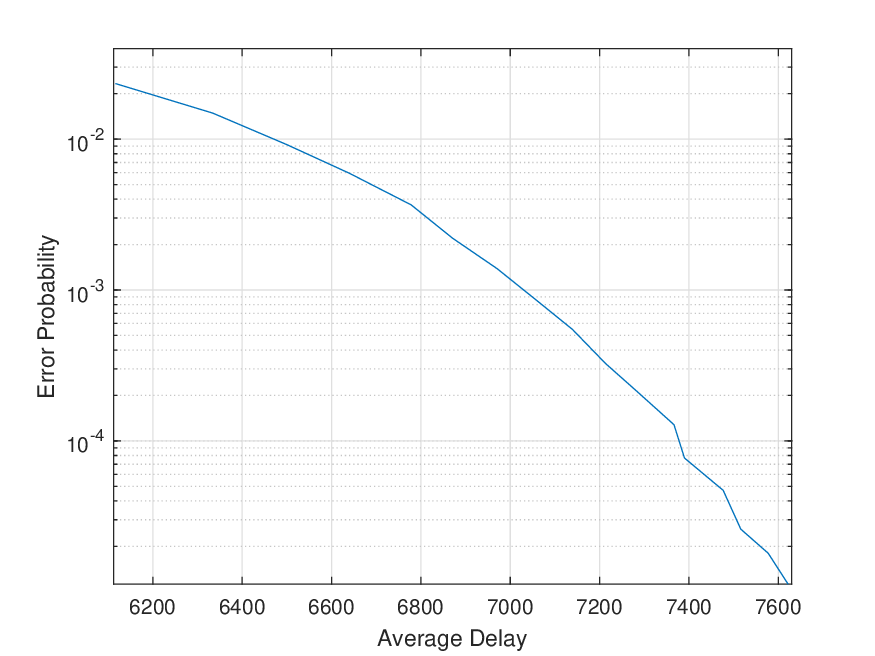}
    \caption{The error probability as a function of the average detection delay, with $M=5$, $\Theta^{(0)}= \{1+0.1N , 1\leq N \leq 10  \}$ , $\Theta^{(1)}= \{0.5+0.1N , 0\leq N \leq 4  \} \cup \{2.1+0.1N , 0\leq N \leq 4  \}$ and $\tau_c=70$.} \label{fig:avg_delay_vs_p_error_tau_c_70}
\end{figure}

We proceed to compare the proposed SCPA algorithm with the CUSUM algorithm, both used for change point detection, with SALLR based on the closest abnormal and normal parameters. At each step, if the SALLR is negative, the next cell is probed. Otherwise, we determine whether to stop and declare the cell anomalous (if the SALLR exceeds \(-\log c\)) or to probe the cell again and repeat the process. The SCPA algorithm was evaluated in two scenarios: without side information and with side information, as detailed in Sections~\ref{ssec:Anomaly Detection Without Side Information} and~\ref{ssec:Anomaly Detection Under a Known Model of Normality}, respectively. The results are presented in Fig.~\ref{fig:comprasion_with_CUSUM} and
Fig.~\ref{fig:second comprasion_with_CUSUM}. In both simulations with side information, the SCPA algorithm demonstrated a clear advantage, achieving a smaller error probability and a shorter detection time from the start of the measurements. Without side information, CUSUM achieved faster detection times for higher error probability values. However, asymptotically, as the error probability decreased, the SCPA algorithm without side information exhibited superior performance.

\begin{figure}[h]
    \centering
    \includegraphics[scale=0.4]{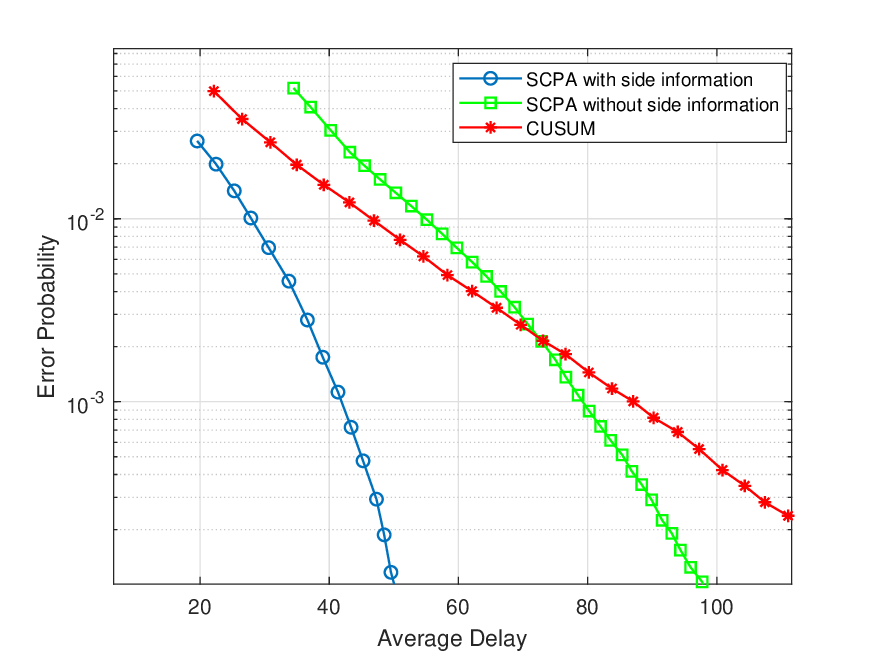}
    \caption{The error probability as a function of the average detection delay under the SCPA algorithm with side information, the SCPA algorithm without side information and the CUSUM algorithm, with $M=4$, $\Theta^{(0)}= \{0.1N , 1\leq N \leq 9  \}$ , $\Theta^{(1)}= \{N , 1\leq N \leq 30 \}$ and $\tau_c=20$.} \label{fig:comprasion_with_CUSUM}
\end{figure}

\begin{figure}[h]
    \centering
    \includegraphics[scale=0.4]{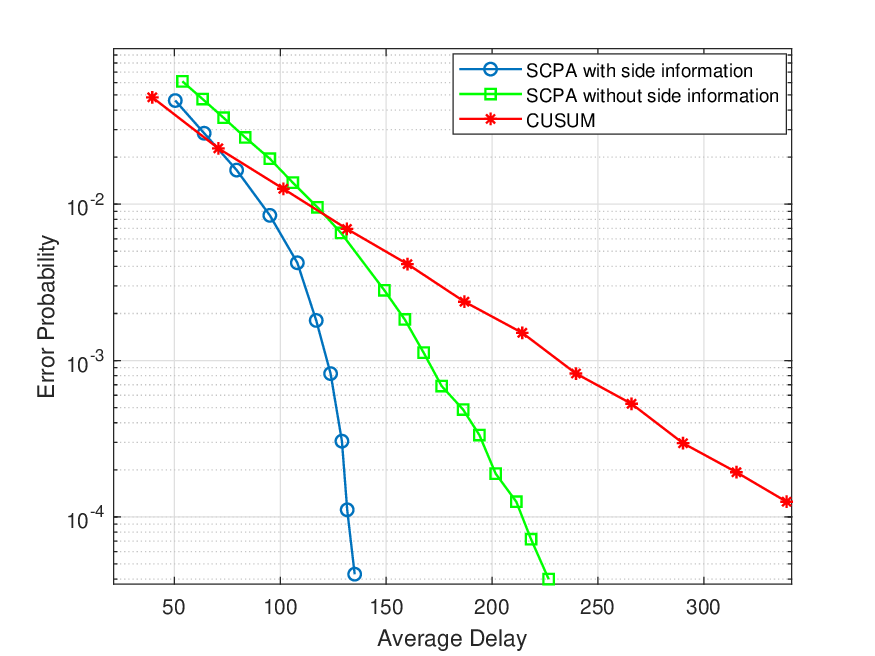}
    \caption{The error probability as a function of the average detection delay under the SCPA algorithm with side information, the SCPA algorithm without side information and the CUSUM algorithm, with $M=7$, $\Theta^{(0)}= \{10 + 10N , 1\leq N \leq 9  \}$ , $\Theta^{(1)}= \{0.1 + 0.5N , 0\leq N \leq 37 \}$ and $\tau_c=0$.} \label{fig:second comprasion_with_CUSUM}
\end{figure}

\section{Conclusion}

We addressed the anomaly detection problem in a set of \( M \) processes, where only a subset of cells can be probed at each step. Observations depend on unknown cell parameters, classified as normal or abnormal. Before the change point, all cells are normal; after the change point, one cell transitions to the abnormal state. Our goal was to develop a sequential strategy minimizing detection time under an error probability constraint. We analyzed two scenarios: without prior information and with known normal parameters. We proved asymptotic optimality in both cases and showed improved detection times with additional information. Simulations validated the theoretical guarantees and efficiency of the proposed SCPA algorithm.

\section{Appendix}
For clarity, we begin by proving Theorem \ref{thorem 2}, followed by highlighting the key steps required to extend the results to the context of Theorem \ref{thorem 1}.

\subsection{Proof of Theorem \ref{thorem 2}}
\label{subsection proof of theorem 2}

To prove the theorem, we introduce the following definitions:

\begin{definition}
Let $\tau_{EST}$ be the minimal integer such that $\tau_{EST}\geq \tau_c$ and for all $n \geq \tau_{EST}$, the parameter estimates satisfy: $\hat{\theta}_{m}(n) = \theta ^{(1)}$, $\hat{\theta}_{j}(n)=\theta^{(0)}$ for all $j \neq m$ under hypothesis $H_{m}$.
\end{definition}

\begin{definition}
$n_{EST} \triangleq \tau_{EST} - \tau_c$ denotes the total amount of time between $\tau_{EST}$ and the change point, $\tau_c$.
\end{definition}

\begin{remark}
Later in the proof, we will show that the expected time spent during the exploration phase is \(O(1)\). Note that for all \(n > \tau_{EST}\), we probe only cell \(m\), which occurs during the exploitation phase. Therefore, the time spent in the exploration phase after the change point is bounded by \(n_{EST}\). As a result, to show that the expected exploration phase time is \(O(1)\), it is sufficient to prove that the expectation of $n_{EST}$ is \(O(1)\). The detailed proof is provided rigorously later.
\end{remark}

\begin{remark}
In the following lemma, when we say that SCPA is implemented indefinitely, we mean that the algorithm operates without stopping, disregarding the stopping rule.
\end{remark}

\begin{lemma}
\label{lemma 1}
Assume that SCPA is implemented indefinitely. Then, there exist $C>0$ and $\gamma>0$ such that
\begin{equation}
\textbf{P}_{m} (n_{EST}>n) \leq Ce^{-\gamma \sqrt{n}}.
\end{equation}
\end{lemma}
\begin{proof}
At any given moment, we are either in Phase 1 (exploration) or Phase 2 (exploitation). Therefore, if $n_{EST} > n$, it implies that after the change point, we remained in either Phase 1 or Phase 2 for at least $n/2$ time indices. Let $T_1$ denote the time spent in Phase 1 after $\tau_c$. Then, we have:
    \begin{align}
    \label{lemma 1 proof eq1}
        \mathbf {P}_{m}(n_{EST}>n) &\leq \mathbf {P}_{m}\left(n_{EST}>n \; , \; T_{1}\geq \frac{n}{2} \right)
        \notag \\&+
        \mathbf {P}_{m}\left(n_{EST}>n \; , \; T_{1} < \frac{n}{2} \right).
    \end{align}

We now bound the first term on the RHS of (\ref{lemma 1 proof eq1}). A complete round of Phase 1 requires $M$ time indices. Therefore, there are at least $\frac{n}{2M}$ rounds of Phase 1 after the anomaly occurs. Each round of Phase 1 can lead to one of the following outcomes: (i) Return to Phase 1. (ii) Move to Phase 2 with $\hat{m}(t)=m$. (iii) Move to Phase 2 with $\hat{m}(t) \neq m$. Based on the above, at least $\frac{n}{6M}$ rounds end with one of these outcomes. Thus, it suffices to demonstrate that the probability of each outcome is bounded.

\noindent
\textit{(i) For at least $\frac{n}{6M}$ rounds that ended by returning to Phase 1:} If a round of Phase 1 ends with a return to Phase 1, it implies either that the observations ended by: $\hat{\theta}_{j} \in \Theta^{(0)}$ for all $j$, or that there exist two cells with $\hat{\theta}_{j} \not\in \Theta^{(0)}$ (so we have $j \neq m$ with $\hat{\theta}_{j} \not \in \Theta^{(0)}$). Following the same arguments as before, we deduce that after the change point, we have at least $\frac{n}{12M}$ rounds of Phase 1 that end with $\hat{\theta}_{j} \in \Theta^{(0)}$ for all $j$, or at least $\frac{n}{12M}$ rounds of Phase 1 that end with some $j \neq m$ that $\hat{\theta}_{j} \not \in \Theta^{(0)}$. We will show that each of them is bounded.

If there are at least $\frac{n}{12M}$ rounds ending with $\hat{\theta}_{j}(t) \in \Theta^{(0)}$ for all $j$, it implies that there are at least $\frac{n}{12M}$ observations of cell $m$ resulting in $\hat{\theta}_{m}(t) \in \Theta^{(0)}$. 
Lets $N_{m}(n)$ be the total number of these observations and $t_{1},...,t_{N_{m}(n)}$ represent their corresponding time indices. Therefore, by the definition of the MLE in (\ref{eq: MLE estimation in exploration}) we know that $\ell_{m}^{(\theta^{(1)},\hat{\theta}_{m}(t_{i}))} (t_{i})<0$ whenever $\hat{\theta}_{m}(t_{i}) \neq \theta^{(1)}$, for all $1 \leq i \leq N_{m}(n)$, and by summation we have,
    \begin{equation}
        \sum_{i=1}^{N_{m}(n)} \ell_{m}^{(\theta^{(1)},\hat{\theta}_{m}(t_{i}))} (t_{i})<0.
    \end{equation}
Thus, applying the Chernoff bound and leveraging the independence of $\ell_{m}^{(\theta^{(1)},\hat{\theta}_{m}(t_{i}))} (t_{i})$, we obtain:
    \begin{align}
        &\displaystyle
        \mathbf {P}_{m}\left( \sum_{i=1}^{N_{m}(n)} \ell_{m}^{(\theta^{(1)},\hat{\theta}_{m}(t_{i}))} (t_{i}) < 0 \right) \notag \\ &\leq \mathbf \prod_{i=1}^{N_{m}(n)}{E}_m\left[e^{-s\cdot\ell_{m}^{(\theta^{(1)},\hat{\theta}_{m}(t_{i}))}(t_{i})}\right]
    \end{align}
    for all $s>0$.
    
Note that a moment-generating function (MGF) is equal to 1 at $s=0$. Since $\textbf{E}_m\left[ -\ell_{m}^{(\theta^{(1)},\hat{\theta}_{m}(t_{i}))}(t_{i})\right] = - \displaystyle{D}(\theta ^{(1)}|\hat{\theta}_{m}(t_{i}))<0$ is strictly negative, the derivative of the MGF of 
$-\ell_{m}^{(\theta^{(1)},\hat{\theta}_{m}(t_{i}))}(t_{i})$ with respect to $s$ is strictly negative at $s=0$. Therefore, there exist $s_{i}>0$ and $\gamma_{i}>0$ such that $\textbf{E}_m\left[e^{-s_{i}\cdot\ell_{m}^{(\theta^{(1)},\hat{\theta}_{m}(t_{i}))}(t_{i})}\right] < e^{-\gamma_{i}}$. Consequently, we can find $s>0$ and $\gamma>0$ such that $\textbf{E}_m\left[e^{-s\cdot\ell_{m}^{(\theta^{(1)},\hat{\theta}_{m}(t_{i}))}(t_{i})}\right] < e^{-\gamma}$ for all $\hat{\theta}_{m}(t_{i})\in \Theta^{(0)}$. Hence,
    \begin{align}
        \mathbf \prod_{i=1}^{N_{m}(n)}\textbf{E}_m\left[e^{-s\cdot\ell_{m}^{(\theta^{(1)},\hat{\theta}_{m}(t_{i}))}(t_{i})}\right] \leq e^{-\gamma \cdot N_{m}(n)} \notag \\ \leq e^{-\gamma'n} \leq e^{-\gamma'\sqrt{n}},
    \end{align}
where the last inequality holds because $N_{m}(n)\geq\frac{n} {12M}$, and for an appropriately chosen $\gamma'$, we obtain the desired result.

For at least $\frac{n}{12M}$ rounds of Phase 1, some $j \neq m$ satisfies $\hat{\theta}_{j}(t) \not \in \Theta^{(0)}$. Note that this $j$ may vary across rounds, but there are at least $\frac{n}{12M^{2}}$ rounds where the same $j$ occurs. Consequently, we have $N_{j}(n) \geq \frac{n}{12M^{2}}$ observations of $j$ that resulted in $\hat{\theta}_{j}(t) \neq \theta^{(0)}$. Let $t_{1}, \dots, t_{N_{j}(n)}$ denote the time indices of these observations, and $\hat{\theta}_{j}(t_{i})$ represent their corresponding estimates during Phase 1. By the MLE definition (\ref{eq: MLE estimation in exploration}), we have $\ell_{j}^{(\theta^{(0)}, \hat{\theta}_{j}(t_{i}))}(t_{i}) < 0$ whenever $\hat{\theta}_{j}(t_{i}) \neq \theta^{(0)}$, for all $1 \leq i \leq N_{j}(n)$, which implies:
    
    \begin{equation}
        \sum_{i=1}^{N_{j}(n)} \ell_{j}^{(\theta^{(0)},\hat{\theta}_{j}(t_{i}))}(t_{i})<0.
    \end{equation}

By applying the Chernoff bound and leveraging the independence of $\ell_{j}^{(\theta^{(0)},\hat{\theta}_{j}(t_{i}))}(t_{i})$, we obtain:
    \begin{align}
    \label{proof lemma 1 eq}
        &\displaystyle
        \mathbf {P}_{m}\left( \sum_{i=1}^{N_{j}(n)} \ell_{j}^{(\theta^{(0)},\hat{\theta}_{j}(t_{i}))}(t_{i})<0 \right) \notag \\ &\leq \mathbf \prod_{i=1}^{N_{j}(n)}\textbf{E}_m\left[e^{-s\cdot\ell_{j}^{(\theta^{(0)},\hat{\theta}_{j}(t_{i}))}(t_{i})}\right]
    \end{align}
    For all $s>0$. Since $\textbf{E}_m\left[ -\ell_{j}^{(\theta^{(0)},\hat{\theta}_{j}(t_{i}))}(t_{i})\right] = - \displaystyle D(\theta^{(0)}|\hat{\theta}_{j}(t_{i}) )<0$ for all $1 \leq i \leq N_{j}(n)$, by applying similar arguments as in the previous case, we can find $s>0$ and $\gamma>0$ such that $\textbf{E}_m\left[e^{-s\cdot\ell_{m}^{(\theta^{(0)},\hat{\theta}_{j}(t_{i}))}(t_{i})}\right] < e^{-\gamma}$ for all $\hat{\theta}_{j}(t_{i}) \not \in \Theta^{(0)}$. Hence, (\ref{proof lemma 1 eq}) can be bounded by:
    \begin{align}
        e^{-\gamma \cdot N_{j}(n)} &\leq e^{-\gamma'n} \leq e^{\gamma'\sqrt{n}},
    \end{align}
    for some suitable $\gamma'$.

    \textit{(ii) For at least $\frac{n}{6M}$ rounds that ended by moving to Phase 2 with $\hat{m}(t)=m $:} 
Since $\tau_{EST} > n$, we are guaranteed to return to Phase 1. This means that the observations of cell $m$ in Phase 2 result in the MLE, as defined in (\ref{eq: MLE estimation in exploitation}), being within the normal parameter space, $\Theta^{(0)}$. 
Denote by $\mathcal{N}_{2}(n)$ the number of times we enter Phase 2 under these conditions up to time $n$. Therefore, we have $\mathcal{N}_{2}(n)\geq\frac{n}{6M}$ under this assumption. 
For $1 \leq i \leq \mathcal{N}_{2}(n)$ let $T_{i}$ be the time at which we exit the exploration phase (as defined in the exploitation phase), and $N_{i}$ be the time at which we exit Phase 2 and return to Phase 1. Thus, for all $1 \leq i \leq \mathcal{N}_{2}(n)$ we have:
    \begin{equation}
    \label{proof lemma 1 eq for case ii}
        \hat{\theta}_{m}(N_{i}) = \arg \max _{\theta \in \Theta} f(\overline{\textbf{y}}_{m}(N_{i})|\theta) \in \Theta^{(0)},
    \end{equation}
    when $\hat{\theta}_{m}(N_{i}) \neq \theta^{(1)}$ and $\overline{\textbf{y}}_{m}(N_{i}) = \{y_{m}(T_{i}+1),..., \\y_{m}(N_{i})\}$.
    Therefore, (\ref{proof lemma 1 eq for case ii}) implies:
    \begin{equation}
        \sum_{t=T_{i}+1}^{N_{i}} \ell_{m}^{(\theta^{(1)},\hat{\theta}_{m}(N_{i}))}(t)<0.
    \end{equation}
    By summing over all $i$, we obtain:
    \begin{equation}
        \sum_{i=1}^{\mathcal{N}_{2}(n)}\left( \sum_{t=T{i}+1}^{N_{i}} \ell_{m}^{(\theta^{(1)},\hat{\theta}_{m}(N_{i}))}(t)\right)<0.
    \end{equation}
    By applying the Chernoff bound and the i.i.d property of $\ell_{j}^{(\theta^{(1)},\hat{\theta}_{m}(N_{i}))}$ we obtain:
    \begin{align}
        &\displaystyle
        \mathbf {P}_{m}\left(  \sum_{i=1}^{\mathcal{N}_{2}(n)} \sum_{t=T{i}+1}^{N_{i}} \ell_{m}^{(\theta^{(1)},\hat{\theta}_{m}(N_{i}))}(t)<0 \right)\notag \\ &\leq \mathbf \prod_{i=1}^{\mathcal{N}_{2}(n)} \left[\textbf{E}_m\left(e^{-s\cdot\ell_{m}^{(\theta^{(1)},\hat{\theta}_{m}(N_{i}))}(T_{i}+1)}\right) \right]^{N_{i}-T_{i}}.
    \end{align}
    Since $\textbf{E}_m\left[ -\ell_{m}^{(\theta^{(1)},\hat{\theta}_{m}(N_{i}))}(T_{i}+1)\right] = - \displaystyle D(\theta^{(1)}|\hat{\theta}_{m}(N_{i}) )<0$ for all $1 \leq i \leq \mathcal{N}_{2}(n)$, there exist $s_{i}>0$ and $\gamma_{i}>0$ such that $\textbf{E}_m\left[e^{-s_{i}\cdot\ell_{m}^{(\theta^{(1)},\hat{\theta}_{m}(N_{i}))}(T_{i}+1)}\right]<e^{-\gamma_{i}}$. Thus, we can find $s>0$ and $\gamma>0$ such that $\textbf{E}_m\left[e^{-s\cdot\ell_{m}^{(\theta^{(1)},\hat{\theta}_{m}(N_{i}))}(T_{i}+1)}\right]<e^{-\gamma}$ for all $\hat{\theta}_{m}(N_{i})$. As a result, we have:
    \begin{align}
        &\displaystyle
         \mathbf \prod_{i=1}^{\mathcal{N}_{2}(n)} \left[\textbf{E}_m\left(e^{-s\cdot\ell_{m}^{(\theta^{(1)},\hat{\theta}_{m}(N_{i}))}(T_{i}+1)}\right) \right]^{N_{i}-T_{i}}\notag \\ &\leq  \mathbf \prod_{i=1}^{\mathcal{N}_{2}(n)} e^{-\gamma \cdot(N_{i}-T_{i})} \leq e^{-\gamma \mathcal{N}_{2}(n)} \leq e^{-\gamma'n} \notag \\ &\leq e^{-\gamma \sqrt{n}},
    \end{align}
    where the second inequality follows from $N_{i}-T_{i}\geq 1$.

    \textit{(iii) For at least $\frac{n}{6M}$ rounds that ended by moving to Phase 2 with $\hat{m}(t)\neq m $:}

If we moved to Phase 2 with $\hat{m}(n) \neq m$, this implies that at least $\frac{n}{6M}$ observations of cell $m$ resulted in $\hat{\theta}_{m}(t) \in \Theta^{(0)}$. Therefore, the probability is bounded using the same arguments as in the first case of (i) in step 1.\vspace{0.2cm}

Next, we bound the second term on the RHS of (\ref{lemma 1 proof eq1}). We can enter Phase 2 in one of two ways: (i) $\hat{m}(t) = m$, or (ii) $\hat{m}(t) \neq m$. Spending less than $\frac{n}{2}$ time indices in Phase 1 implies that we spent more than $\frac{n}{2}$ time indices in Phase 2. In this case, we either spent more than $\frac{n}{4}$ time indices in Phase 2 with $\hat{m}(t) = m$, or we spent more than $\frac{n}{4}$ time indices in Phase 2 with $\hat{m}(t) \neq m$. It is sufficient to show that each of these cases is bounded.
    
\textit{(i) For more than $\frac{n}{4}$ time on Phase 2 with $\hat{m}(t)=m$:} As in (ii) in step 1, denote by $\mathcal{N}_{2}(n)$ the number of times we entered Phase 2 with cell $m$ up to time $n$. Since $\tau_{EST} > n$, we know that we will return to Phase 1. Let $T_{i}$ and $N_{i}$ represent the time we entered and the time we returned to Phase 1, respectively, for all $1 \leq i \leq \mathcal{N}_{2}(n)$. Notice that $\sum_{i=1}^{\mathcal{N}_{2}(n)}(T_{i} - N_{i})$ is the total time spent in Phase 2 with cell $m$ up to time $n$, which is assumed to be greater than $\frac{n}{4}$. Using similar arguments as in (ii) in step 1, there exists $\gamma > 0$ such that the probability of this case is bounded by:
    \begin{align}
        \mathbf \prod_{i=1}^{\mathcal{N}_{2}(n)} e^{-\gamma \cdot(N_{i}-T_{i})} = & e^{-\gamma \sum_{i=1}^{\mathcal{N}_{2}(n)}(T_{i}-N_{i})} < e^{-\gamma \cdot \frac{n}{4}} \notag \\ &\leq e^{-\gamma' n} \leq e^{-\gamma' \sqrt{n}}. 
    \end{align}     

    \textit{(ii) For more than $\frac{n}{4}$ time on Phase 2 with $\hat{m}(t)\neq m$:} 
    
    Under this assumption, one of the following occurs: (a) We have more than $\frac{\sqrt{n}}{2}$ rounds of Phase 1 that end with moving to Phase 2 where $\hat{m}(t) \neq m$. (b) We enter Phase 2 with $\hat{m}(t) \neq m$ and remain in Phase 2 for more than $\frac{\sqrt{n}}{2}$ time indices. If neither (a) nor (b) occurs, we will conclude that we entered Phase 2 with the wrong cell at most $\frac{\sqrt{n}}{2}$ times, and in each case, we stayed for at most $\frac{\sqrt{n}}{2}$ time indices. Therefore, the total time spent in Phase 2 with the wrong cell is at most $\frac{\sqrt{n}}{2} \cdot \frac{\sqrt{n}}{2} = \frac{n}{4}$, which contradicts the assumption in (ii). Hence, to prove the lemma, it remains to show that both (a) and (b) are bounded.

For Case (a), we have more than $\frac{\sqrt{n}}{2}$ observations of cell $m$ that result in $\hat{\theta}_{m}(t) \in \Theta^{(0)}$. Therefore, following the same reasoning as in the first case of (i) in step 1, let $N_{m}(n)$ denote the number of these observations up to time $n$, with the condition that $N_{m}(n) > \frac{\sqrt{n}}{2}$. Using the Chernoff bound and the properties of $\ell_{m}^{(\theta^{(1)}, \hat{\theta}_{m}(t))}$ (independence and negative expectation), we can find a constant $\gamma > 0$ such that the probability of this case is bounded by:
    \begin{align}
        e^{-\gamma \cdot N_{m}(n)} < e^{-\gamma \frac{\sqrt{n}}{2} }= e^{-\gamma' \sqrt{n}}.
    \end{align}

For Case (b), if we spent more than $\frac{\sqrt{n}}{2}$ time units in Phase 2 with $\hat{m} \neq m$, this implies that after $\frac{\sqrt{n}}{2}$ time units in Phase 2, we did not return to Phase 1. Let $T$ denote the time at which we entered Phase 2. Thus, at time $t = T + \frac{\sqrt{n}}{2}$, the MLE, as defined in (\ref{eq: MLE estimation in exploitation}), satisfies $\hat{\theta}_{\hat{m}}(t) \not \in \Theta^{(0)}$. Therefore,
    \begin{equation}
        \sum_{r=T+1}^{t} \ell_{\hat{m}}^{(\theta^{(0)},\hat{\theta}_{\hat{m}(t)})}(r)<0.
    \end{equation}
    By applying the Chernoff bound and using the i.i.d property of $\ell_{\hat{m}}^{(\theta^{(0)},\hat{\theta}_{\hat{m}(t)})}(r)$, we have:
    \begin{align}
    \label{eq: lemma 1 ii b}
        &\displaystyle
        \mathbf {P}_{m}\left( \sum_{r=T+1}^{t} \ell_{\hat{m}}^{(\theta^{(0)},\hat{\theta}_{\hat{m}(t)})}(r)<0 \right) \notag \\ &\leq
        \left[\textbf{E}_m\left(e^{-s\cdot\ell_{\hat{m}}^{(\theta^{(0)},\hat{\theta}_{\hat{m}}(t))}(T+1)}\right) \right] ^{\frac{\sqrt{n}}{2}}.
    \end{align}
    Using the fact that $\textbf{E}_m\left( \ell_{\hat{m}}^{(\theta^{(0)},\hat{\theta}_{\hat{m}}(t))}(T+1) \right) =  -\displaystyle D(\theta^{(0)}||\hat{\theta}_{\hat{m}}(t) )<0$, we can find $s>0$ and $\gamma>0$, such that (\ref{eq: lemma 1 ii b}) is bounded by
    \begin{align}
        e^{-\gamma \cdot \frac{\sqrt{n}}{2}} = e^{-\gamma' \sqrt{n}},
    \end{align}
    for $\gamma'=\gamma/2$, which completes the proof of the lemma.
\end{proof}

\begin{definition}
Let $\tau_{U}$ be the smallest integer such that $S_{m}(n) \geq -\log(c)$ for all $n>\tau_{EST}$:
\begin{equation}
    \tau_{U} \triangleq \inf \{ n>\tau_{EST} : S_{m}(n) \geq -\log(c) \},
\end{equation}
and let $n_{U} \triangleq \tau_{U} - \tau_{EST}$ denote the total amount of time between $\tau_{EST}$ and $\tau_{U}$.    
\end{definition}

\begin{lemma}
\label{lemma 1}
    Assume that SCPA is implemented indefinitely. Then, for every fixed $\epsilon>0$ there exist $C>0$ and $\gamma>0$ such that
\begin{equation}
\begin{array}{l}
\displaystyle
\mathbf {P}_m\left(n_{U} >n\right)\leq C e^{-\gamma \sqrt{n}}\vspace{0.1cm}\\ \hspace{2cm}
\displaystyle
\forall n >-(1+\epsilon)\log c/D\left(\theta ^{(1)}||\theta ^{(0)}\right).    
\end{array}
\end{equation}

\begin{proof}
        Let 
        \begin{align}
            \ell_{m}(t) \triangleq \ell_{m}^{(\hat{\theta}_{m}(t-1),\theta^{(0)})}(t)  =  \log \frac{f(y_{m}(t)|\hat{\theta}_{m}(t-1))}{f(y_{m}(t))|{\theta}^{(0)})}
        \end{align}
        and,
        \begin{align}
            \Tilde{\ell}_{m}(t) \triangleq \ell_{m}(t) - D(\theta^{(1)}||\theta^{(0)})
        \end{align}
denote the ALLR and the normalized ALLR, respectively, of cell $m$ at time $t$. Note that for all $t > \tau_{EST}$, $\Tilde{\ell}_{m}(t)$ is a zero-mean random variable under hypothesis $H_{m}$, and as such, we refer to it as the normalized ALLR. 

For $n > \tau_{EST}$, by the definition of $\tau_{EST}$ and in the context of the exploitation phase notation, we have $T \leq \tau_{EST}$. Let $\epsilon_{1} = D(\theta^{(1)}||\theta^{(0)})\epsilon / (1+\epsilon)>0$. Then,         \begin{align}
        \label{lemma 2 proof eq 1}
            &\displaystyle
            \sum_{t=T+2}^{\tau_{EST}+n}\ell_{m}(t) + \log{c} = 
            \sum_{t=T+2}^{\tau_{EST}}\ell_{m}(t) + \sum_{\tau_{EST}+1}^{\tau_{EST}+n}\ell_{m}(t) + \log c  
            \notag \\ & = \sum_{t=T+2}^{\tau_{EST}}\ell_{m}(t) + \sum_{\tau_{EST}+1}^{\tau_{EST}+n}\Tilde{\ell}_{m}(t)+ nD(\theta^{(1)}||\theta^{(0)}) + \log c  
            \notag \\ &\geq \sum_{t=T+2}^{\tau_{EST}}\ell_{m}(t) + \sum_{\tau_{EST}+1}^{\tau_{EST}+n}\Tilde{\ell}_{m}(t)+ n\epsilon_{1},
        \end{align}
        for all $n>-(1+\epsilon)\log c / D(\theta^{(1)}|\theta^{(0)}).$
        
        \vspace{0.1cm}
        As a result, $\sum_{t=T+2}^{\tau_{EST}+n}\ell_{m}(t) \leq -\log{c}$, implies that 
\begin{center}
$\sum_{t=T+2}^{\tau_{EST}}\ell_{m}(t) + \sum_{\tau_{EST}+1}^{\tau_{EST}+n}\Tilde{\ell}_{m}(t) \leq  -n\epsilon_{1}$.    
\end{center}
        Therefore, for every $\epsilon>0$ there exists $\epsilon_{1}>0$ such that,
        \begin{align}
        \label{eq : lemma 2 eq 1}
        &\displaystyle
        \mathbf {P}_{m}\left( \sum_{t=T+2}^{\tau_{EST}+n}\ell_{m}(t) \leq -\log{c} \right) 
        \notag \\ & \leq 
        \mathbf {P}_{m}\left( \sum_{t=T+2}^{\tau_{EST}}\ell_{m}(t) + \sum_{\tau_{EST}+1}^{\tau_{EST}+n}\Tilde{\ell}_{m}(t) \leq  -n\epsilon_{1} \right)
        \notag \\ & \leq 
        \mathbf {P}_{m}\left( \sum_{t=T+2}^{\tau_{EST}}\ell_{m}(t) \leq  -n\epsilon_{1}/2 \right)
        \notag \\ & + 
        \mathbf {P}_{m}\left( \sum_{\tau_{EST}+1}^{\tau_{EST}+n}\Tilde{\ell}_{m}(t) \leq  -n\epsilon_{1}/2 \right).
        \end{align}
        Hence, it suffices to show that each term on the RHS of (\ref{eq : lemma 2 eq 1}) is bounded.\vspace{0.2cm}

Next, we bound the first term on the RHS of (\ref{eq : lemma 2 eq 1}). Fix $q>0$. Then,
        \begin{align}
        \label{eq : lemma 2 eq 2}
        &\displaystyle
        \mathbf {P}_{m}\left( \sum_{t=T+2}^{\tau_{EST}}\ell_{m}(t) \leq  -n\epsilon_{1}/2 \right)
        \notag \\ & \leq 
        \mathbf {P}_{m}\left( \sum_{t=T+2}^{\tau_{EST}}\ell_{m}(t) \leq  -n\epsilon_{1}/2 \;  , \; n_{EST}> qn \right)
        \notag \\ & + 
        \mathbf {P}_{m}\left( \sum_{t=T+2}^{\tau_{EST}}\ell_{m}(t) \leq  -n\epsilon_{1}/2 \;  , \; n_{EST}\leq qn \right).    
        \end{align}
        Again, it suffices to prove that each term on the RHS of (\ref{eq : lemma 2 eq 2}) is bounded.
        The first term is bounded by Lemma 1. For the second term, we have two cases: (i) The last exit from Phase 1 occurred after the change point, i.e., $T \geq \tau_c$. (ii) The last exit from Phase 1 occurred before the change point, i.e., $T<\tau_c$. Under Case (i), $\tau_{EST}-T \leq n_{EST} \leq qn$. Since $q>0$ can be arbitrarily small and $\ell_m(t)$ has a finite expectation, by applying the Chernoff bound we obtain that the probability of Case (i) decreases exponentially with $n$. Hence, we have $C>0$ and $\gamma>0$ such that the probability of (i) is bounded by $Ce^{-\gamma \sqrt{n}}$. We fix again $r>0$. Then, (ii) is bounded by:
        \begin{align}
        \label{eq : lemma 2 eq 4}
        \mathbf {P}_{m} \bigg( &\displaystyle \sum_{t=T+2}^{\tau_{EST}}\ell_{m}(t) \leq  -n\epsilon_{1}/2 \;  , \; n_{EST} \leq qn \; , \notag  \\ & \tau_c-T \leq rn \; , \; \; T < \tau_c \bigg) \notag \\
        + \;\; \mathbf {P}_{m} \bigg( &\displaystyle \sum_{t=T+2}^{\tau_{EST}}\ell_{m}(t) \leq  -n\epsilon_{1}/2 \;  , \; n_{EST} \leq qn \; , \notag  \\ & \tau_c-T > rn \; , \; \; T < \tau_c \bigg) \notag \\ \leq
         \mathbf {P}_{m} \bigg( &\displaystyle \sum_{t=T+2}^{\tau_{EST}}\ell_{m}(t) \leq  -n\epsilon_{1}/2 \;  , \; \tau_{EST}-T\leq (q+r)n \bigg)\notag \\
        + \;\; \mathbf {P}_{m} \bigg( &\displaystyle \sum_{t=T+2}^{\tau_{EST}}\ell_{m}(t) \leq  -n\epsilon_{1}/2 \;  , \tau_c-T > rn \bigg).
        \end{align}
The first term on the RHS of (\ref{eq : lemma 2 eq 4}) can be bounded using similar arguments as in Case (i). For any $q > 0$ and $r > 0$, arbitrarily small values can be chosen, and $\ell_m(t)$ has a finite expectation. For the second term, we observe that after the time index $\tau_c$, there was no return to Phase 1, i.e., $\hat{\theta}_m(\tau_c) \not \in \Theta^{(0)}$. Consequently, the observations $y_m(T+1), \dots, y_m(\tau_c)$ resulted in $\hat{\theta}_m(\tau_c) \neq \theta^{(0)}$. By the MLE definition in (\ref{eq: MLE estimation in exploration}), we have:
        \begin{align}
            \sum_{t=T+1}^{\tau_c} \ell_m^{(\theta^{(0)} , \hat{\theta}_m(\tau_c))}<0.
        \end{align}
Hence, by applying the Chernoff bound and using the i.i.d property of $\ell_m^{(\theta^{(0)} , \hat{\theta}_m(\tau_c))}(t)$, we obtain:
        \begin{align}
        \label{eq : lemma 2 eq 5}
            &\displaystyle \mathbf {P}_{m} \bigg(  \sum_{t=T+1}^{\tau_{EST}}\ell_{m}(t) \leq  -n\epsilon_{1}/2 \;  , \tau_c-T > rn \bigg)
            \notag \\
            & \leq \mathbf{E}_{m} \left[ e^{-s \cdot \ell_m^{(\theta^{(0)} , \hat{\theta}_m(\tau_c))}(T+1)} \right] ^{\tau_c-T},
        \end{align}
        for all $s>0$. Since $\mathbf{E}_{m} \left( \ell _m^ {(\theta^{(0)} , \hat{\theta}_m(\tau_c))}(T+1) \right) = \displaystyle D(\theta^{(0)}||\hat{\theta}_m(\tau_c))>0$, we can find $\gamma>0$ suvh that (\ref{eq : lemma 2 eq 5}) is bounded by: 
        \begin{align}
            e^{-\gamma \cdot (\tau_c-T)} \leq e^{-\gamma \cdot rn} \leq e^{-\gamma' n} \leq e^{-\gamma ' \sqrt{n}},
        \end{align}
        for $\gamma '=\gamma \cdot r >0$.\vspace{0.2cm}

Next, we Bound the second term on the RHS of (\ref{eq : lemma 2 eq 1}). Note that $\Tilde{\ell}_m(t)$ has zero mean for all $t>\tau_{EST}$. Therefore, by applying the Chernoff bound and using the i.i.d property of $\Tilde{\ell}_m(t)$ we have:
\begin{equation}
\begin{array}{l}
        \label{eq : lemma 2 eq 6}
    \displaystyle\mathbf {P}_{m}\left( \sum_{\tau_{EST}+1}^{\tau_{EST}+n}\Tilde{\ell}_{m}(t) \leq  -n\epsilon_{1}/2 \right) \\
    \displaystyle\mathbf {E}_{m} \left[ e^{-s\cdot \sum_{t=\tau_{EST}+1}^{\tau_{EST} +n} (\Tilde{\ell}_m(t) +\epsilon_1) } \right]  \leq \mathbf{E}_{m} \left[ e^{-s \cdot (\Tilde{\ell}_m(t) +\epsilon_1)} \right]^n,
\end{array}    
\end{equation}
for all $s>0$. Hence, there exist $\gamma>0$ such that (\ref{eq : lemma 2 eq 6}) is bounded by $e^{-\gamma n}$, which completes the proof of the lemma.
    \end{proof}
\end{lemma}

\begin{lemma}
The error probability under SCPA is $O\left(c \cdot (- \log c)^{1-\delta}\right)$.  
\end{lemma}
    \begin{proof}
        Recall that the error probabuility is given by:
        \begin{equation} P_{e} =  \sum_{m=1}^{M} \pi_{m} \alpha_{m}. \end{equation}
        Let $\alpha_{m,j} ^{MD} = \mathbf {P}_{m}(\delta = j , \, \tau > \tau_c)$ for all $j \neq m$, and $\alpha_{m,j} ^{FA} = \mathbf {P}_{m}(\delta = j , \, \tau \leq \tau_c)$ for all $j$, denote the miss-detect and false-alarm probabilities, respectively. Thus, $\alpha_{m} \vphantom{\bigg[} = \sum_{j \neq m}^{M} \alpha_{m,j} ^{MD} + \sum_{j=1}^{M} \alpha_{m,j} ^{FA}$.
        
        Therefore, it suffices to show that $\alpha_{m,j}^{MD}\, ,\,  \alpha_{m,j}^{FA}= O\left(c \cdot (- \log c)^{1-\delta}\right)$. We start by analyzing the false-alarm error. Note that         
        \begin{align} 
        \alpha _{m,j} ^{FA}&=\mathbf {P}_{m}(\delta = j , \, \tau \leq \tau_c)
        \notag \\&=
        \displaystyle \mathbf {P}_{m}\left (S_{j}(\tau )\geq -\log c \;\; \text{for some} \; \tau \leq \tau_c \right ) 
        \notag \\ &\leq
        \sum_{T=1}^{\tau_c} \mathbf \, {P}_{m} \bigg(  \sum_{t=T+2}^{\tau} \log \frac{f(y_{j}(t)| (\hat{\theta}_{j}(t-1))}{f(y_{j}(t)|\theta^{(0)})} \geq -\log c  \notag \\&\qquad  \qquad \qquad
        \text{for some} \; T < \tau \bigg)
        \notag \\ &\leq 
        \sum_{T=1}^{\tau_c} \mathbf \, {P}_{m} \left( Z_T(\tau - T) \geq \frac{1}{c} \;\; \text{for some} \; T < \tau \right),
        \end{align}
        where we define for fixed $T \geq 1$,
        \begin{equation} Z_T(\tau-T) \triangleq e^{{S_{j}(\tau)}} = \prod_{t=T+2}^{\tau} \frac{f(y_{j}(t)|\hat{\theta}_{j}(t-1))}{f(y_{j}(t)|\theta^{(0)})} \;.  \end{equation}
        Also, notice that $Z_T(\tau-T)$ is a non-negative martingale: 
        \begin{align} 
        &\displaystyle {\textbf{E}}_{m} \left[Z_T\left(\tau -T\right)|\left\lbrace y_j(t)\right\rbrace _{t=T+2}^{\tau-1}\right] 
        \notag \\ &\quad \displaystyle =Z_T(\tau-1-T){\textbf{E}}_{m}\left[\frac{f(y_j(\tau)|\hat{\theta }_j{(\tau-1)})}{f(y_j(\tau)|\theta ^{(0)})}\right] 
        \notag \\ &\quad \displaystyle =Z_T\left(\tau -1 -T\right)\!. \end{align}
        Hence, by Lemma 1 in \cite{robbins1972class} for a non-negative martingale, we have,
        \begin{align} 
        \alpha _{m,j} ^{FA}&=\sum_{T=1}^{\tau_c} c \cdot \displaystyle {\textbf{E}}_{m} \left[Z_T\left( 1 \right) \right] = c \cdot \tau_c \leq c \cdot (-\log c)^{1-\delta},
        \end{align}
        where the last equality is due to the fact that ${{\textbf{E}}}_{m} \left[Z_T\left( 1 \right) \right]=1$, since for $\tau \leq \tau_c$ , $\theta_j(\tau) = \theta^{(0)}$.

Next, we analyze the miss-detect error. Note that 
        \begin{align} 
        \alpha _{m,j} ^{MD}&=\mathbf {P}_{m}(\delta = j , \, \tau > \tau_c)
        \notag \\&=
        \displaystyle \mathbf {P}_{m}\left (S_{j}(\tau )\geq -\log c \;\; \text{for some} \; \tau > \tau_c \right )
        \notag \\ &= 
        \sum_{r=\tau_c+1}^{\infty} \mathbf \, {P}_{m} \bigg(  \sum_{t=r+2}^{\tau} \log \frac{f(y_{j}(t)| (\hat{\theta}_{j}(t-1))}{f(y_{j}(t)|\theta^{(0)})} \geq -\log c  \notag \\&\qquad  \qquad \qquad
        \text{for some} \; r < \tau \; , \; T=r \bigg)
        \notag \\ &= 
        \sum_{r=\tau_c+1}^{\infty} \mathbf \, {P}_{m} \bigg(  \sum_{t=r+2}^{\tau} \log \frac{f(y_{j}(t)| (\hat{\theta}_{j}(t-1))}{f(y_{j}(t)|\theta^{(0)})} \geq -\log c  \notag \\&\qquad 
        \text{for some} \; r < \tau \; | \; T=r \bigg) \cdot \mathbf \, {P}_{m} \left( T=r \right) 
        \notag \\ &= 
        \sum_{r=\tau_c+1}^{\infty} \mathbf \, {P}_{m} \bigg(  \sum_{t=r+2}^{\tau} \log \frac{f(y_{j}(t)| (\hat{\theta}_{j}(t-1))}{f(y_{j}(t)|\theta^{(0)})} \geq -\log c  \notag \\&\qquad 
        \text{for some} \; r < \tau \bigg) \cdot \mathbf \, {P}_{m} \left( T=r \right) 
        \notag \\ &\leq 
        \sum_{r=\tau_c+1}^{\infty} \mathbf \, {P}_{m} \left( Z_r(\tau - r) \geq \frac{1}{c} \;\; \text{for some} \; r < \tau \right) \notag \\&\qquad \qquad \cdot \mathbf \, {P}_{m} \left( T=r \right)
        \notag \\ &\leq 
        \sum_{r=\tau_c+1}^{\infty} \, c \cdot {P}_{m} \left( T=r \right)
        \end{align}
        The last inequality arises from the fact that the event $\{T = r\}$ depends solely on samples up to $t = r$, making it independent of samples from $t = r + 1$ onward. Additionally, the inequality follows from Lemma 1 in \cite{robbins1972class} for non-negative martingales, applied to $Z_r(\tau - r)$, along with the property that ${{\textbf{E}}}_{m} \left[Z_r\left( 1 \right)\right] = 1$ for all $j \neq m$ and fixed $r$. Finally, note that $T \leq \tau_{EST}$. Therefore, we can apply Lemma 1 and obtain:
        \begin{align} 
        \alpha _{m,j} ^{MD}&\leq         \sum_{r=\tau_c+1}^{\infty} c \cdot {P}_{m} \left( T=r \right) \leq \sum_{r=\tau_c}^{\infty} c \cdot {P}_{m} \left( T>r \right)
        \notag \\ &\leq 
        \sum_{r=\tau_c}^{\infty} c \cdot {P}_{m} \left( \tau_{EST}>r \right) \leq \sum_{n=0}^{\infty} c \cdot {P}_{m} \left( n_{EST}>n \right)
        \notag \\ &= O(c) = O\left(c \cdot (- \log c)^{1-\delta}\right),
        \end{align}
        which completes the proof of the lemma.
    \end{proof}

Next, we complete the proof of the theorem. From Lemma 1, we have $\mathbf{E}_{m}[n_{EST}] \leq O(1)$, and from Lemma 2, we have $\mathbf{E}_{m}[n_{U}] \leq -(1+o(1))\frac{\log(c)}{D(\theta^{(1)}||\theta^{(0)})} \vphantom{\bigg[}$. 

The detection time $\tau$ under the SCPA algorithm is upper bounded by $\tau_{U}$. Therefore, $\tau - \tau_c \leq \tau_U - \tau_c = n_U + n_{EST}$. Combining all together, we obtain:
\begin{align}
\label{eq: proof of theorem 2}
    \mathbf{E}_{m}[\tau-\tau_c] \leq -(1+o(1)) \frac{\log(c)}{D(\theta^{(1)}||\theta^{(0)})}.
\end{align}
To upper bound the Bayes risk under the SCPA algorithm, we apply the bound on the error by Lemma 3, and (\ref{eq: proof of theorem 2}) to have 
\begin{align}
    R_m(\Gamma) \leq -(1+o(1)) \frac{c \log(c)}{D(\theta^{(1)}||\theta^{(0)})}.
\end{align}
Lastly, we apply the lower bound on the Bayes risk under simple hypotheses for any algorithm \cite{cohen2015active}: $R_m(\Gamma) \geq -(1+o(1)) \frac{c \log(c)}{D(\theta^{(1)}||\theta^{(0)})}$. Combining the upper and lower bounds concludes the proof.

\subsection{Proof of Theorem 1}
\label{subsection proof of theorem 1}
In this section, we extend the proof for the case when no additional information on the parameters is given, and for all $m$, $\theta_m$ is unknown. Without loss of generality, we prove the theorem when hypothesis $m$ is true. 


To bound the detection time, we define $\tau_{U}$ and $n_{U}$ as in Definition 3, where $\tau_U \triangleq \inf \{ n > \tau_{EST} : S_m(n) \geq -\log(c) \}$ and $n_U \triangleq \tau_U - \tau_{EST}$. Additionally, we introduce the following definition:  

\begin{definition}
For any $\varphi \in \Theta^{(0)}$, $\tau_{U}(\varphi)$ is the smallest integer satisfying $\sum_{t=T+2}^{n} \ell^{(\theta^{(1)},\varphi)}_{m}(t) \geq -\log(c)$ for $n > \tau_{EST}$. Specifically,  
\begin{align}
    \tau_{U}(\varphi) \triangleq \inf \{ n > \tau_{EST} : \sum_{t=T+2}^{n} \ell^{(\theta^{(1)},\varphi)}_{m}(t) \geq -\log(c) \}.
\end{align}
Also, let $n_{U}(\varphi) \triangleq \tau_U(\varphi) - \tau_{EST}(\varphi)$ denote the total amount of time between $\tau_{EST}$ and $ \tau_U(\varphi)$.  
\end{definition}

Note that for all $n>\tau_{EST}$, $S_m(n) = \sum_{t=T+2}^{n} \ell^{(\theta^{(1)},\theta^{(0)})}_{m}(t)$ and $n_U = \max_{\varphi \in \Theta^{(0)}} \{n_U(\varphi)\} $. The following lemma bounds $n_U(\varphi)$ for all $\varphi \in \Theta^{(0)}$.

\begin{lemma}
    Assume that SCPA is implemented indefinitely. Then, for every $\varphi \in \Theta^{(0)}$ and for every fixed $\epsilon>0$ there exist $C>0$ and $\gamma>0$ such that 
    \begin{align}
    \label{eq lemma 3}
        &\mathbf {P}_m\left(n_{U}(\varphi) >n\right)\leq C e^{-\gamma n} \notag \\ &\quad\hspace{2cm}\forall n >-(1+\epsilon)\log c/D(\theta ^{(1)}) \;.
    \end{align}
    \begin{proof}
    As in the first case, we define:
    \begin{align}
        \Tilde{\ell}^{(\theta^{(1)},\varphi)}_{m}(t) \triangleq \ell^{(\theta^{(1)},\varphi)}_{m}(t) - D(\theta^{(1)}||\varphi).
    \end{align}
    Again, under hypothesis $H_m$, $\Tilde{\ell}_m(t)$ has zero mean for all $t>\tau_{EST}$. Let $\epsilon_1 = D(\theta^{(1)}||\varphi)\epsilon (1+\epsilon)>0$ and as in (\ref{lemma 2 proof eq 1}), we have:
    \begin{align}
        &\sum_{t=T+2}^{\tau_{EST}+n} \ell^{(\theta^{(1)},\varphi)}(t) + \log c \hphantom{TTTTTTTTTTTTTTTTTTTT} \notag \\
        & \geq \sum_{t=T+2}^{\tau_{EST}} \ell^{(\theta^{(1)},\varphi)}(t) + 
        \sum_{t=\tau_{EST}+1}^{\tau_{EST}+n} \Tilde{\ell}^{(\theta^{(1)},\varphi)}(t) + n\epsilon_1 ,
    \end{align}
    for all $n>-(1+\epsilon)\log c / D(\theta^{(1)}||\varphi)$.
    Since $D(\theta^{(1)}||\varphi)\geq D(\theta^{(1)})$ for all $\varphi \in \Theta^{(0)}$, we can prove (\ref{eq lemma 3}) using the same steps as in the proof of Lemma 2.
    \end{proof}
\end{lemma}

Next, we show that the error probability is $O\left(c \cdot (- \log c)^{1-\delta}\right)$. Similar to the previous case, it suffices to prove that $\alpha_{m,j}^{MD} \,,\,  \alpha_{m,j}^{FA} = O\left(c \cdot (- \log c)^{1-\delta}\right)$ for all $j \neq m$ and $j$, respectively. Note that for $j \neq m$, $\theta_j \in \Theta^{(0)}$ and for $\tau<\tau_c$ , $\theta_j \in \Theta^{(0)}$ for all $j$. Therefore,
\begin{align}
    & S_{j}(\tau) = \sum_{t=T+2}^{\tau} \frac{f(y_j(t)|\hat{\theta}_j(t-1))}{f(y_j(t)|\hat{\theta}^{(0)}_j(\tau))}    
    \notag \\ 
    & \hspace{0.85cm}=\min_{\varphi \in \Theta^{(0)}} \sum_{t=T+2}^{\tau} \frac{f(y_j(t)|\hat{\theta}_j(t-1))}{f(y_j(t)|\varphi)} \notag \\ 
    & \hspace{0.85cm}\leq \sum_{t=T+2}^{\tau} \frac{f(y_j(t)|\hat{\theta}_j(t-1))}{f(y_j(t)|\theta_j)}. 
\end{align}
Hence, with the same notations as in the previous case for fixed $T$ , $Z_T(\tau-T) \triangleq e^{{S_{j}(\tau)}} = \prod_{t=T+2}^{\tau} \frac{f(y_{j}(t)|\hat{\theta}_{j}(t-1))}{f(y_{j}(t)|\theta_j)} \;,$ we have:
\begin{align}
    \alpha_{m,j}^{FA} &= \displaystyle \mathbf {P}_{m}\left( \delta=j \,,\, \tau \leq \tau_c \right)
    \notag \\ & =\displaystyle \mathbf {P}_{m}\left (S_{j}(\tau )\geq -\log c \;\; \text{for some} \; \tau \leq \tau_c \right ) 
    \notag \\ & \leq \sum_{T=1}^{\tau_c} \displaystyle \mathbf {P}_{m}\bigg( \sum_{t=T+2}^{\tau} \frac{f(y_j(t)|\hat{\theta}_j(t-1))}{f(y_j(t)|\theta_j)}\geq -\log c \;\; \notag \\ & \hspace{3cm}\text{for some} \; T<\tau \bigg)
    \notag \\ & \leq \sum_{T=1}^{\tau_c} \displaystyle \mathbf {P}_{m}\left (Z_T(\tau-T) \geq \frac{1}{c} \;\; \text{for some} \; T<\tau \right ) 
    \notag \\  & \leq \sum_{T=1}^{\tau_c} c \, \mathbf{E}_{m} [Z_T(1)]  = c \cdot \tau_c \leq c \cdot \left( -\log c \right)^{1-\delta},
\end{align}
where again we used Lemma 1 in \cite{robbins1972class} for a non-negative martingale, $Z_T(\tau-T)$. The last inequality follows due to $\mathbf {E}_{m} [Z_T(1)]=1$. The same holds for $\alpha_{m,j}^{MD}$, which proves the error bound.

To complete the proof, note that $n_U>n$ implies that there exists $\varphi \in \Theta^{(0)}$ with $n_U(\varphi)>n$, such that
\begin{align}
    \label{bounding n_U for case 2}
    &\mathbf {P}_m\left( n_U>n \right) \leq  \sum_{\varphi \in \Theta^{(0)}} \mathbf {P}_m\left( n_U(\varphi)>n \right).
\end{align}
Recall that $\Theta^{(0)}$ is a finite space. Therefore, using Lemma 4, for every $\epsilon>0$ we can find $C>0$ and $\gamma>0$ such that for all $\varphi  \in \Theta^{(0)}$  we have: 
\begin{align}
    \label{lemma 2 in case 2}
    \mathbf {P}_m\left( n_U(\varphi)>n \right) \; , \; \forall n>-(1+\epsilon) \log c/D(\theta^{(1)}).
\end{align}
Using (\ref{lemma 2 in case 2}), we obtain that (\ref{bounding n_U for case 2}) is bounded by:
\begin{align}
    \Tilde{C}e^{-\gamma \sqrt{n}} \; , \; \forall  n>-(1+\epsilon) \log c/D(\theta^{(1)}),
\end{align}
for suitable $\Tilde{C}>0$. Since $\tau-\tau_c \leq n_{EST} + n_{U}$, using the bound on $n_{EST}$ and $n_U$, we have:
    \begin{align}
    \label{bound on detection time second case}
        \mathbf{E}_{m}[\tau-\tau_c] \leq -(1+o(1)) \frac{\log(c)}{D(\theta^{(1)})}.
    \end{align}
    Finally, combining the error bound and (\ref{bound on detection time second case}) yields:
    \begin{align}
        \label{bound on the Bayes risk second case}
        R_m(\Gamma) \leq -(1+o(1)) \frac{c \log(c)}{D(\theta^{(1))})},
    \end{align}
 which, together with the lower bound on the Bayes risk, completes the proof.

\bibliographystyle{ieeetr}
%

\end{document}